\documentclass[opre,nonblindrev]{informs4a}
\OneAndAHalfSpacedXI

\usepackage{mathtools,bm}
\usepackage{aliascnt}
\usepackage{booktabs,tabularx}
\usepackage{enumitem}
\usepackage[nopatch=footnote]{microtype}
\usepackage{natbib}
\usepackage[hyperfootnotes=false]{hyperref}
\usepackage[capitalize,noabbrev]{cleveref}
\usepackage{algorithm}
\usepackage[noend]{algpseudocode}
\usepackage{float}
\usepackage{needspace}
\usepackage{graphicx}

\bibpunct[, ]{(}{)}{,}{a}{}{,}
\def\bibfont{\footnotesize\linespread{1}\selectfont}

\hypersetup{
 hypertexnames=false,
 colorlinks=true,
 linkcolor=blue,
 citecolor=blue,
 urlcolor=blue,
 filecolor=blue,
 pdftitle={Resource-Adaptive Stochastic Gradient Descent for Online Linear Programming without Re-solving},
 pdfauthor={Jiameng Lyu},
}
\renewcommand{\arraystretch}{1.12}
\setlist{leftmargin=2em,itemsep=0.1em,topsep=0.2em}

\newtheorem{theorem}{Theorem}[section]
\newaliascnt{lemma}{theorem}
\newtheorem{lemma}[lemma]{Lemma}
\aliascntresetthe{lemma}
\newaliascnt{proposition}{theorem}

\aliascntresetthe{proposition}
\newaliascnt{corollary}{theorem}
\newtheorem{corollary}[corollary]{Corollary}
\aliascntresetthe{corollary}
\newaliascnt{assumption}{theorem}
\newtheorem{assumption}[assumption]{Assumption}
\aliascntresetthe{assumption}
\newaliascnt{remark}{theorem}

\aliascntresetthe{remark}
\numberwithin{equation}{section}
\crefname{theorem}{Theorem}{Theorems}
\Crefname{theorem}{Theorem}{Theorems}
\crefname{lemma}{Lemma}{Lemmas}
\Crefname{lemma}{Lemma}{Lemmas}
\crefname{proposition}{Proposition}{Propositions}
\Crefname{proposition}{Proposition}{Propositions}
\crefname{corollary}{Corollary}{Corollaries}
\Crefname{corollary}{Corollary}{Corollaries}
\crefname{assumption}{Assumption}{Assumptions}
\Crefname{assumption}{Assumption}{Assumptions}
\crefname{remark}{Remark}{Remarks}
\Crefname{remark}{Remark}{Remarks}

\newcommand{\R}{\mathbb R}
\newcommand{\E}{\mathbb E}
\newcommand{\Prob}{\mathbb P}
\newcommand{\1}{\mathbf 1}
\newcommand{\F}{\mathcal F}
\newcommand{\Y}{\mathcal Y}
\newcommand{\B}{\mathcal B}
\newcommand{\N}{\mathcal N}
\newcommand{\G}{\mathcal G}
\newcommand{\OPT}{\operatorname{OPT}}
\newcommand{\Reg}{\operatorname{Reg}}
\newcommand{\norm}[1]{\left\lVert #1\right\rVert}
\newcommand{\ip}[2]{\left\langle #1,#2\right\rangle}
\newcommand{\RA}{\mathsf{RASGD}}
\newcommand{\epsT}{\varepsilon_T}

\newcommand{\dmin}{d_{\min}}
\newcommand{\dmax}{d_{\max}}
\newcommand{\mulo}{\mu_0}
\DeclareRobustCommand{\algname}[1]{\mbox{\sffamily\upshape #1}}
\newcolumntype{P}[1]{>{\raggedright\arraybackslash}p{#1}}
\newcolumntype{Z}{>{\raggedright\arraybackslash}X}

\begin{document}

\RUNTITLE{Resource-Adaptive SGD for Online Linear Programming}
\TITLE{Resource-Adaptive Stochastic Gradient Descent for Online Linear
Programming without Re-solving}
\RUNAUTHOR{Lyu}
\ARTICLEAUTHORS{
\AUTHOR{Jiameng Lyu}
\AFF{Department of Management Science, School of Management, Fudan University,
Shanghai 200433, China, \EMAIL{jiamenglyu@fudan.edu.cn}}
}

\ABSTRACT{
The growth of large language model (LLM) inference and search services increases the scale of online linear programming problems, motivating computationally efficient algorithms.
We develop resource-adaptive stochastic gradient descent (RASGD) for stochastic online linear programming.
The algorithm uses one request and current inventory to update resource prices, requiring \(O(m)\) operations for \(m\) resources and memory per arrival and no LP or sample-average optimization.
The central idea is to express the current-resource pricing logic of re-solving through a first-order SGD update: each arrival refreshes the remaining-inventory allowance in the dual objective, while the stepsize decreases for early learning and increases later to match the speed of inventory adjustment.
Under standard non-degeneracy conditions, our algorithm is feasible on every sample path and achieves \(O(\log T)\) expected regret against the realized fractional hindsight optimum, which matches the lower bound, even for policies that know the distribution and have unrestricted computation.
The analysis converts curvature around the fixed reference price into inventory stability without tracking optimal prices at changing resource levels.
Numerical experiments show that RASGD achieves regret competitive with per-arrival LP re-solving and improves upon the tested first-order baselines, while retaining the computational efficiency of first-order methods. These results establish RASGD as a computationally
efficient approach to achieving high allocation quality in large-scale OLP.
}


\maketitle

\section{Introduction}\label{sec:intro}
Online linear programming (OLP) allocates limited resources to sequentially arriving requests
\citep{AgrawalWangYe2014,Jasin2015,BumpensantiWang2020,LiSunYe2020,LiYe2022}.
Applications include online advertising, network revenue management, and cloud computing.
LLM inference also admits an OLP formulation when each request reveals its value and reserved token or compute allowance upon arrival.
In these settings, resource-price learning guides decisions made before future demand is known.

Computation becomes critical at scale. Google reports more than five trillion searches annually \citep{Google2025},
and LLM serving must sustain high throughput under latency and GPU-memory constraints \citep{Kwon2023}.
Repeated large LP solves can therefore be costly in the admission path, motivating infrequent re-solving and LP-free first-order methods
\citep{BalseiroLuMirrokni2020,LiSunYe2020}.
For finite-support inputs, \citet{LiWangZhang2026} obtain constant regret with \(O(\log\log T)\) LP solves.

We study a class permitting continuous rewards. Under non-degeneracy conditions,
\citet{Bray2025} establishes matching logarithmic upper and lower bounds, but the unknown-distribution policy attaining the upper bound solves an empirical LP every period.
Algorithm~5 of \citet{Ma2025} avoids LP solves and proves \(O(\log^2 T)\) regret using epoch-wise resource adaptation,
under growth and response conditions uniform over resource levels.
\citet{Gao2026} remove this resource-uniform requirement and obtain an \(O(T^{1/3})\) first-order guarantee in the continuous, nondegenerate setting.
We propose an algorithm framework that attains logarithmic regret with exact feasibility under standard non-degeneracy conditions.
Its defining feature is the combination of per-arrival resource feedback and a stepsize matched to inventory dynamics.

\subsection{Contributions}
We make the following three contributions.
\begin{enumerate}
\item \textbf{A resource-adaptive stochastic gradient descent framework.}
We propose RASGD, a first-order framework that couples per-arrival resource feedback with stochastic price updates through a current-resource dual objective.
The key design coordinates per-arrival resource adaptation with a single
horizon-dependent stepsize: decreasing steps support early learning, while
increasing steps near the horizon match the faster inventory dynamics.
Together, these components implement the current-resource pricing logic
of re-solving through stochastic gradient updates, without epoch restarts
or active-set identification. RASGD requires \(O(m)\) operations and memory
per arrival for \(m\) resources, solves no LP or sample-average optimization
problem, and produces feasible binary decisions on every sample path.

\item \textbf{Optimal logarithmic regret under standard non-degeneracy.}
We prove \(O(\log T)\) expected regret against the realized fractional
hindsight optimum under standard non-degeneracy conditions. The response condition is anchored at the initial optimal price, without requiring uniform growth and response bounds across resource levels. An embedding of a known multisecretary lower bound
establishes the optimal horizon order within this class, even with a known
distribution and unrestricted computation. The analysis explains how the
joint design achieves this guarantee: per-arrival feedback and the matched
late stepsize turn curvature at a fixed reference price into restoring
inventory drift. A joint early analysis controls the initial error, while
strict slack controls nonbinding prices. This yields logarithmic regret
through price accuracy and inventory stability without tracking a moving
optimizer.

\item \textbf{Strong empirical performance with first-order efficiency.}
Experiments across resource levels, horizons, and binding structures
demonstrate the practical effectiveness of RASGD. In representative
single-resource and multiple-resource benchmarks, it achieves allocation
quality close to that of state-of-the-art per-arrival LP re-solving and
lower regret than the other tested first-order methods. Its runtime
remains comparable to these first-order methods while being orders of
magnitude lower than per-arrival LP re-solving in the tested
implementations. These results establish RASGD as a computationally
efficient approach to achieving high allocation quality in large-scale OLP.
\end{enumerate}

\subsection{Related Literature}\label{sec:related}
\textbf{Learning and re-solving in online resource allocation.}
LP-based control and re-solving have been extensively studied in network revenue management, both with known demand distributions and with demand learning
\citep{JasinKumar2012,Jasin2015,ChenLiYe2024}.
\citet{AgrawalWangYe2014} use geometric price-learning intervals to obtain near-optimal competitive guarantees in the random-order model.
\citet{LiYe2022} connect empirical dual optimization to a population stochastic program and prove \(O(\log T\log\log T)\) regret for action-history-dependent learning under non-degeneracy conditions with resource-uniform growth and response bounds.
\citet{Bray2025} establishes logarithmic upper and lower bounds for continuous-valuation multisecretary and OLP models.
\citet{ChenWang2025} show that, with a known arrival distribution, the standard certainty-equivalent policy achieves \(O((\log T)^2)\) hindsight regret, including instances with fluid degeneracy.
Related work addresses continuous rewards under degeneracy, random consumption, and gapped multisecretary instances
\citep{JiangMaZhang2025,ZhangContinuous2026,ZhangBellman2026}.
\citet{ChenZhouEtAl2026} extend OLP to settings with stochastic resource replenishment.

\textbf{Infrequent re-solving and LP-free methods.}
Computational considerations motivate both reducing the frequency of re-solving and replacing optimization solves with first-order updates.
LP-free first-order methods include the dual mirror descent policy of \citet{BalseiroLuMirrokni2020}, with an \(O(\sqrt T)\) guarantee;
\citet{LiSunYe2020} use one-pass projected stochastic subgradient updates.
Primal--dual learning also applies to resource-constrained revenue management with large action spaces \citep{MiaoWangZhang2026}.

For unknown finite-support inputs, \citet{Jasin2015} obtains
\(O(\log^2 T)\) regret under standard non-degeneracy conditions using
\(O(\log T)\) LP re-solves. \citet{LiWangZhang2026} achieve constant regret with \(O(\log\log T)\) LP re-solves, including degenerate instances.
Their policy combines infrequent re-solving with first-order computations between solves.

For continuous-support inputs, several first-order approaches avoid per-period re-solving.
Algorithm~5 of \citet{Ma2025} updates the gradient's resource target at geometric epoch boundaries.
They established \(O(\log^2 T)\) regret bound under resource-uniform growth and response conditions.
\citet{Gao2026} separate price learning from decisions and obtain an \(O(T^{1/3})\) guarantee under standard non-degeneracy conditions for continuous inputs.

In our OLP setting, RASGD refreshes the resource target every arrival and matches its stepsize to inventory dynamics.
Under standard non-degeneracy conditions, it achieves optimal \(O(\log T)\) hindsight regret for a class permitting continuous rewards, with prefix feasibility and no LP solves.
The analysis controls prices and inventory around a fixed reference price, without requiring uniform growth and response bounds across resource levels.
Table~\ref{tab:comparison} summarizes the input classes, guarantees, and computational requirements.

\begin{table}[H]
\centering
\SingleSpacedXI
\small
\caption{Learning algorithms for stochastic online linear programming with unknown distributions}
\label{tab:comparison}
\medskip
\setlength{\tabcolsep}{4pt}
\renewcommand{\arraystretch}{1.22}
\begin{tabularx}{\textwidth}{@{}P{1.95cm}ZP{2.9cm}P{2.8cm}P{1.85cm}@{}}
\toprule
\textbf{Input class} & \textbf{Reference} & \textbf{Non-degeneracy} & \textbf{Regret / performance} & \textbf{\# of re-solvings}\\
\midrule
General & \citet{BalseiroLuMirrokni2020} & Allows degeneracy & $O(\sqrt T)$ & $0$\\
General & \citet{LiSunYe2020} & Allows degeneracy & $O(\sqrt T)^\dagger$ & $0$\\
\midrule
Finite & \citet{Jasin2015} & Standard & $O(\log^2 T)$ & $O(\log T)$\\
Finite & \citet{Gao2026} & Standard & $O(\log T)^\dagger$ & $0$\\
Finite & \citet{LiWangZhang2026} & Allows degeneracy & $O(1)$ & $O(\log\log T)$\\
\midrule
Continuous & \citet{LiYe2022}, Algorithm~2 & Standard & $O(\sqrt T\log T)$ & $O(\log T)$\\
Continuous & \citet{LiYe2022}, Algorithm~3 & \textbf{Standard + uniformity} & $O(\log T\log\log T)$ & $O(T)$\\
Continuous & \citet{Bray2025} & Standard & $O(\log T)$ & $O(T)$\\
Continuous & \citet{Ma2025}, Algorithm~5 & \textbf{Standard + uniformity} & $O(\log^2 T)$ & $0$\\
Continuous & \citet{Gao2026} & Standard & $O(T^{1/3})^\dagger$ & $0$\\
\textbf{Continuous} & \textbf{This paper} & \textbf{Standard} & $\mathbf{O}(\log T)$ & $\mathbf{0}$\\
\bottomrule
\end{tabularx}
\smallskip
\par\noindent\footnotesize
\emph{Notes:}
All listed policies learn from unknown arrival distributions.
``Standard'' denotes each paper's non-degeneracy conditions; ``Allows degeneracy''
means that no non-degeneracy assumption is required. ``+ uniformity''
indicates additional growth and response bounds that hold uniformly over a
range of resource levels. The precise assumptions differ across papers.
The Li--Ye rows use Assumptions~1--2 (Theorem~4) and
Assumptions~1 and~3 (Theorem~5), respectively.
$\dagger$ denotes guarantees involving both objective loss and constraint
violation.
\end{table}

\textbf{Organization.}
Section~\ref{sec:model} gives the model and assumptions;
Section~\ref{sec:algorithm} presents RASGD;
Section~\ref{sec:analysis} proves its regret guarantee and the lower-bound embedding.
Section~\ref{sec:experiments} reports experiments, and Section~\ref{sec:conclusion} concludes.
The appendices contain supporting proofs and theoretical parameter dependence.

\section{Model and Assumptions}\label{sec:model}
\subsection{Online Allocation and the Hindsight Benchmark}
There are \(m\) resources, a known horizon \(T\), and initial inventory \(B_1=Td\), where \(d\in\R_{++}^m\) is the initial resource level per period.
At arrival \(t\), the policy observes reward \(r_t\) and consumption vector \(a_t\in\R_+^m\), then chooses an irrevocable decision \(x_t\in\{0,1\}\).

A policy \(\pi=(\pi_t)_{t=1}^T\) specifies a decision rule for each arrival.
It is \emph{non-anticipative} if \(x_t\) uses only requests observed up to \(t\)
and internal randomization independent of the request sequence.
It is \emph{feasible} if, on every sample path,
\begin{equation}
 B_t:=Td-\sum_{s<t}a_sx_s\ge0,\qquad 1\le t\le T+1. \label{eq:inventory}
\end{equation}
Throughout, we consider feasible, non-anticipative policies.
We consider the standard packing model: accepting a request consumes nonnegative amounts of resources.
We compare the policy with the fractional allocation that knows the entire realized sequence:
\begin{equation}
 \OPT_T^{\rm H}
 =\max\left\{\sum_{t=1}^T r_t z_t:\ \sum_{t=1}^T a_tz_t\le Td,\ 0\le z_t\le1\right\}.\label{eq:benchmark}
\end{equation}
Following the convention in the literature \citep{LiYe2022}, we define the regret of a policy \(\pi\) relative to this benchmark as
\begin{equation}
 \Reg_T(\pi)=\E\left[\OPT_T^{\rm H}-\sum_{t=1}^T r_tx_t\right].\label{eq:regret}
\end{equation}

The expectation includes arrivals and policy randomization. A prefix-feasible binary allocation is feasible for the hindsight LP, so realized regret is nonnegative.

Write \(\F_{t-1}\) for the information available just before request \(t\), and \(\E_t[\cdot]=\E[\cdot\mid\F_{t-1}]\).

\subsection{Population Prices and Non-degeneracy}
Let \(\mathcal P\) denote the unknown joint distribution of a
generic request \((r,a)\in\R\times\R_+^m\). Expectations and probabilities
concerning a generic request below are taken under \((r,a)\sim\mathcal P\).

A price vector \(y\ge0\) assigns cost \(a^\top y\) to a request's resource bundle. The corresponding threshold rule accepts when the reward covers this cost. Define the population dual objective, preferred action, and mean consumption by
\begin{align}
 f_d(y)&=d^\top y+\E(r-a^\top y)_+,&
 X(y;r,a)&=\1\{r\ge a^\top y\},\label{eq:dual-response}\\
 h(y)&=\E[aX(y;r,a)],&
 F(y)&=d-h(y).\label{eq:field}
\end{align}

At equality we accept, fixing the subgradient selection throughout.
Here \(F(y)\in\partial f_d(y)\) is the mean resource surplus induced by \(X(y;r,a)\).

\needspace{5\baselineskip}
At the initial resource level \(d\), let
\begin{equation}
 \Y^*(d):=\operatorname*{arg\,min}_{y\in\R_+^m}f_d(y),
 \qquad y^*=y^*(d)\in\Y^*(d).
 \label{eq:reference-price}
\end{equation}
The analysis uses \(y^*\) as a fixed reference price.
The algorithm does not know \(y^*\) and does not track optimizers associated with the evolving resource level.
The following assumptions imply \(\Y^*(d)=\{y^*\}\) (Lemma~\ref{lem:geometry}).

\needspace{7\baselineskip}
\begin{assumption}\label{ass:gao}
Let \(y^*=y^*(d)\) be the reference price defined in \eqref{eq:reference-price}.
We impose the following input and non-degeneracy conditions.
\begin{enumerate}[label=\textup{(G\arabic*)},ref=G\arabic*]
\item\label{ass:input}The request pairs satisfy
\((r_t,a_t)\overset{\mathrm{i.i.d.}}{\sim}\mathcal P\),
\(t=1,\ldots,T\). Known finite constants \(\bar r,\bar a>0\) and
\(0<\underline d\le\bar d<\infty\) satisfy
\(|r|\le\bar r\) and \(a\in[0,\bar a]^m\)
\(\mathcal P\)-almost surely, and \(d\in[\underline d,\bar d]^m\).
\item\label{ass:moment} \(\E[aa^\top]\succeq\lambda_0 I_m\), where \(\lambda_0>0\).
\item\label{ass:threshold} For almost every \(a\) and every \(y\in\Xi_1:=\{y\ge0:\norm y\le\bar r/\underline d+1\}\),
\begin{equation}
 \lambda_1|a^\top(y-y^*)|
 \le\left|\Prob(r\ge a^\top y\mid a)-\Prob(r\ge a^\top y^*\mid a)\right|
 \le\lambda_2|a^\top(y-y^*)|,\label{eq:margin}
\end{equation}
where \(0<\lambda_1\le\lambda_2<\infty\).
\item\label{ass:complementarity} The reference solution satisfies strict complementarity: for every \(i\in[m]\),
\begin{equation}
 \begin{aligned}
 F_i(y^*)&=d_i-\E[a_i\1\{r\ge a^\top y^*\}]=0
 &&\text{if }y_i^*>0,\\
 F_i(y^*)&=d_i-\E[a_i\1\{r\ge a^\top y^*\}]>0
 &&\text{if }y_i^*=0.
 \end{aligned}\label{eq:complementarity}
\end{equation}
\end{enumerate}
\end{assumption}

Define the binding and nonbinding resource sets by
\begin{equation}
 \B=\{i:y_i^*>0\},\qquad \N=[m]\setminus\B. \label{eq:sets}
\end{equation}
Thus, by (\ref{ass:complementarity}),
\begin{equation}
 F_\B(y^*)=0,\qquad s_i:=F_i(y^*)>0\quad(i\in\N).
 \label{eq:reference-slack}
\end{equation}
The algorithm does not know these sets. Empty-coordinate vectors have norm zero, and minima over empty index sets are omitted.

\textbf{Interpretation of the assumptions.}
Condition (\ref{ass:input}) specifies the basic input assumptions.
Conditions (\ref{ass:moment})--(\ref{ass:complementarity}) follow the standard
non-degeneracy framework of \citet[Assumption~2]{LiYe2022}: a positive-definite
resource second moment, two-sided conditional threshold-response bounds,
and strict complementarity.
Conditions (\ref{ass:moment})--(\ref{ass:threshold}) provide a coercive and Lipschitz mean response around the fixed reference price \(y^*\), together with a quadratic reward-loss bound (Lemma~\ref{lem:geometry}).
These response conditions also appear in \citet[Example~1]{Gao2026}.
In (\ref{ass:threshold}), the reference price is fixed at \(y^*(d)\), while
the candidate price \(y\) ranges over the entire prescribed domain \(\Xi_1\).
Condition (\ref{ass:complementarity}) imposes strict complementarity on the response induced by \(X\): positive-price resources have zero slack, while zero-price resources have strictly positive slack.
No analogous growth or response bounds are imposed around optimizers
associated with the evolving resource vector \(q_t\).
This is narrower in scope than the resource-uniform conditions in \citet[Assumption~3]{LiYe2022} and \citet[Assumptions~3.1 and~4.1]{Ma2025}.
For each resource level \(d'\) in a prescribed range, those conditions
require growth and response bounds around the corresponding population
optimizer, with constants independent of \(d'\).
Thus, resource uniformity adds a requirement across resource levels, whereas
(\ref{ass:threshold}) varies only the candidate price \(y\in\Xi_1\) relative
to the single reference \(y^*(d)\).

\section{Resource-Adaptive Stochastic Gradient Descent}\label{sec:algorithm}
The RASGD framework integrates price learning, inventory feedback,
and feasible allocation in a single sequential policy. Within this
framework, RASGD takes one stochastic subgradient step for a dual objective indexed by current inventory, with a stepsize that
determines how quickly prices respond to that target. 

We first derive the current-resource
Lagrangian underlying re-solving and its stochastic first-order update. Then, we give the executable policy, and explain the two-sided stepsize.

\subsection{The Resource-Adaptive Dual Objective}\label{sec:feedback}

\textbf{The primal objective and its price representation.}
Consider the population relaxation at the initial resource level \(d\):
\begin{equation}
 \sup_{x(\cdot,\cdot)}\left\{\E[rx(r,a)]:\
       \E[ax(r,a)]\le d,\quad 0\le x(r,a)\le1\right\}.
 \label{eq:population-primal}
\end{equation}
 Here \(x(r,a)\) is measurable. Unlike the online policy, this relaxation constrains only expected consumption.
Its Lagrangian for \(y\ge0\) is
\begin{equation}
 \mathcal L_d(x,y)
 =\E[rx(r,a)]+y^\top\bigl(d-\E[ax(r,a)]\bigr)
 =d^\top y+\E[(r-a^\top y)x(r,a)].
 \label{eq:population-lagrangian}
\end{equation}
 Maximization separates across requests:
\begin{equation}
 X(y;r,a)=\1\{r\ge a^\top y\}
 \in\arg\max_{0\le x\le1}(r-a^\top y)x,\qquad
 \max_{0\le x\le1}(r-a^\top y)x=(r-a^\top y)_+.
 \label{eq:primal-best-response}
\end{equation}
 Choosing \(x=1\) at ties gives the dual problem
\begin{equation}
 \min_{y\ge0} f_d(y),\qquad
 f_d(y)=\sup_{0\le x(\cdot,\cdot)\le1}\mathcal L_d(x,y)
       =d^\top y+\E(r-a^\top y)_+.
 \label{eq:dual-minimization}
\end{equation}

Thus thresholding solves the inner maximization, and prices are updated by descent on \(f_d\).
Under Assumption~\ref{ass:gao}, the reference rule \(X(y^*;r,a)\) is feasible for \eqref{eq:population-primal} and earns \(f_d(y^*)\) (Lemma~\ref{lem:geometry}).

\textbf{Current-resource feedback.}
At the start of period \(t\), define the remaining horizon, resource rate, and capped allowance by
\begin{equation}
 n_t=T-t+1,\qquad q_t=B_t/n_t,\qquad
 q^c_{t,i}=\min\{q_{t,i},d_i+1\}. \label{eq:cap}
\end{equation}
Replacing \(d\) by a resource allowance \(q\) changes only the linear term
of the dual objective:
\begin{equation}
 f_q(y)=q^\top y+\E(r-a^\top y)_+
       =f_d(y)+(q-d)^\top y.
 \label{eq:adaptive-objective}
\end{equation}

The remaining population relaxation has objective \(n_t f_{q_t}(y)\).
RASGD uses \(f_{q_t^c}\) to keep the gradient bounded; the cap changes neither physical inventory nor feasibility.

\textbf{Relation to re-solving.}
For \(t\ge2\), an empirical re-solving policy minimizes over \(y\ge0\)
\begin{equation}
 \widehat f_t(y;q_t)=q_t^\top y+\frac{1}{t-1}\sum_{s=1}^{t-1}(r_s-a_s^\top y)_+.
 \label{eq:resolving-dual}
\end{equation}
RASGD retains the current-resource term but replaces empirical optimization with one fresh-request subgradient of \(f_{q_t^c}\).
It neither stores the empirical objective nor solves for a new optimizer.
The allowance is held fixed during each price step and recomputed from physical inventory before the next arrival.

\textbf{Coupled updates.}
Write \(\widetilde x_t=X(y_t;r_t,a_t)\).
When the cap, projection, and feasibility filter are inactive, price descent and inventory accounting give
\begin{equation}
 \begin{aligned}
 y_{t+1}&=y_t-\alpha_t(q_t-a_t\widetilde x_t),\\
 q_{t+1}&=q_t+\frac{q_t-a_t\widetilde x_t}{n_t-1},
 \qquad t<T.
 \end{aligned}\label{eq:feedback-recursions}
\end{equation}
The second equation follows from \(B_{t+1}=B_t-a_t\widetilde x_t\)
and the loss of one remaining period. The same consumption discrepancy drives both updates: spending above the allowance raises the price and lowers the resource rate; spending below it has the opposite effect.

\subsection{The Executable Policy}\label{sec:executable}

\textbf{What the sample estimates.}
At period \(t\), the current request defines the convex sample loss
\begin{equation}
 \ell_t(y;q_t^c)=(q_t^c)^\top y+(r_t-a_t^\top y)_+.
 \label{eq:sample-loss}
\end{equation}
 Since \(y_t,q_t^c\) are \(\F_{t-1}\)-measurable and the request is independent of the past,
\(\E_t[\ell_t(y;q_t^c)]=f_{q_t^c}(y)\) for each fixed \(y\).

 The preferred action gives, for every \(y\),
\begin{align}
 \ell_t(y;q_t^c)
 &\ge(q_t^c)^\top y+(r_t-a_t^\top y)\widetilde x_t\notag\\
 &=\ell_t(y_t;q_t^c)
   +(q_t^c-a_t\widetilde x_t)^\top(y-y_t).
 \label{eq:sample-subgradient-inequality}
\end{align}
 Thus the computable subgradient
\begin{equation}
 g_t=q_t^c-a_t\widetilde x_t
 \in\partial_y\ell_t(y_t;q_t^c)
 \label{eq:sample-subgradient}
\end{equation}
is valid even at a threshold tie. Taking conditional expectations yields
\begin{equation}
 \E_t[g_t]=q_t^c-h(y_t)
 =F(y_t)+(q_t^c-d)
 \in\partial f_{q_t^c}(y_t).
 \label{eq:conditional-subgradient}
\end{equation}

The gradient is taken with respect to price, holding \(q_t^c\) fixed.
It uses the preferred action \(\widetilde x_t\), not the feasible action \(x_t\): the full request reveals preferred consumption even when rejected for feasibility.
Using \(x_t\) would generally invalidate \eqref{eq:conditional-subgradient} near depletion.

\textbf{Projection, stepsizes, and execution.}
 Let \(\dmin=\min_i d_i\), \(\dmax=\max_i d_i\), \(D_0=\bar r/\dmin\), and define
\begin{equation}
 \sigma_0=\min\{1,D_0/2\},\qquad
 R=D_0+\sigma_0,\qquad
 \Y=\{y\ge0:\norm y\le R\}. \label{eq:Y}
\end{equation}

Projection takes positive parts, then rescales if the norm exceeds \(R\), using \(O(m)\) operations.
Since \(d^\top y^*\le f_d(y^*)\le f_d(0)\le\bar r\), every minimizer of \(f_d\) satisfies \(\norm{y^*}_1\le D_0\).
Thus \(\Y\subseteq\Xi_1\) contains all minimizers of \(f_d\) with spherical margin at least \(\sigma_0\).

Choose the gain and stepsize
\begin{equation}
 \kappa=\frac{2}{\mulo},\qquad s_0=64,\qquad
 \alpha_t=\frac{\kappa}{\max\{s_0,\min(t,T-t)\}}
 \quad(t<T). \label{eq:clock}
\end{equation}
 Algorithm~\ref{alg:rafo} uses this schedule without restarts and stores only inventory and prices, requiring \(O(m)\) operations and memory per arrival.

\textbf{Interpretation of the stepsize calibration input.}
The policy receives a conservative curvature bound
\begin{equation}
 0<\mulo\le\lambda_0\lambda_1. \label{eq:calibration}
\end{equation}
Any fixed valid positive lower bound is sufficient; a smaller value changes the gain and regret constant, but not the logarithmic horizon order.
Such inverse-curvature calibration is standard in online gradient methods and strongly convex stochastic optimization \citep{HazanAgarwalKale2007,RakhlinShamirSridharan2012}, and is also used in the first-order allocation method of \citet{Ma2025}.
The distribution, reference price, binding set, and slack margins remain unknown; the guarantee assumes a valid supplied bound.
\begin{algorithm}[H]
\caption[Resource-Adaptive Stochastic Gradient Descent (RASGD)]{Resource-Adaptive Stochastic Gradient Descent (\(\RA\))}\label{alg:rafo}
\begin{algorithmic}[1]
\Require \(T,d,\bar r,\mulo\); form \(\Y,\kappa,s_0\) by \eqref{eq:Y} and \eqref{eq:clock}
\State \(B_1\gets Td,\quad y_1\gets0\)
\For{\(t=1,\ldots,T\)}
 \State \(n_t\gets T-t+1,\quad q_t\gets B_t/n_t,\quad q^c_t\gets\min\{q_t,d+\mathbf1\}\)
 \State Observe \((r_t,a_t)\); set \(\widetilde x_t\gets\1\{r_t\ge a_t^\top y_t\}\)
 \State \(x_t\gets\widetilde x_t\,\1\{a_t\le B_t\text{ componentwise}\}\)
 \State \(B_{t+1}\gets B_t-a_tx_t\)
 \If{\(t<T\)}
  \State \(g_t\gets q^c_t-a_t\widetilde x_t\)
  \State \(y_{t+1}\gets\Pi_\Y(y_t-\alpha_tg_t)\)
 \EndIf
\EndFor
\end{algorithmic}
\end{algorithm}

\subsection{A Stepsize Matched to Inventory Dynamics}\label{sec:stepsize}
Early in the horizon, inventory rates change by \(O(1/T)\) per request, while the stepsize is \(\kappa/t\) for \(s_0\le t\le T/2\).
This supports price learning while retaining resource feedback.
Later, \eqref{eq:clock} matches the inventory scale in \eqref{eq:feedback-recursions}:
\[
 \alpha_t=\frac{\kappa}{n_t-1}
 \qquad\text{for }t\ge T/2\text{ and }n_t-1\ge s_0.
\]
This matching yields the binding price--inventory invariant used in the proof.
A stepsize that continued to decrease as \(1/t\) would instead remain of order \(1/T\) while inventory rates became increasingly sensitive.
The truncation \(s_0\) bounds the first and last steps; the midpoint does not trigger a reset or require active-set information.

\subsection{Illustration of Resource Adaptation and Stepsize Compensation}\label{sec:illustration}
Consider one resource with initial rate \(d=0.5\), unit consumption
\(a_t=1\), and i.i.d. rewards \(r_t\sim\operatorname{Unif}[0,1]\).
A price \(y\in[0,1]\) induces mean consumption \(h(y)=1-y\), which is also the acceptance probability because each acceptance consumes one unit.
The inventory-based target is \(q_t=B_t/n_t\): the available resource per remaining arrival.
Thus \(q_t\) specifies the consumption budget, while \(h(y_t)\) is the rate induced by the current price.
We consider states where the cap, projection, and feasibility filter are inactive.

\textbf{Resource adaptation: setting the consumption target.} For a fixed \(q_t\in(0,1)\), the scalar population relaxation maximizes reward at the threshold \(y=1-q_t\), for which \(h(y)=q_t\).
A higher threshold forgoes positive-reward requests despite available capacity; a lower threshold exceeds the mean resource budget.
Matching the rates is therefore the population benchmark for using the available inventory while selecting the highest-value requests that the budget permits.

At \(y_t=0.5\), the induced acceptance rate is \(h(y_t)=0.5\).
With 100 arrivals remaining, inventories of 70 and 30 give targets \(q_t=0.7\) and \(q_t=0.3\), respectively.
Keeping the rate at \(0.5\) would plan to consume 50 units: 20 fewer than the first inventory, but 20 more than the second.
These discrepancies create the risk of unused capacity or early depletion.
The price drift
\[
 \E_t[y_{t+1}-y_t]=\alpha_t\bigl(h(y_t)-q_t\bigr)
\]
is \(-0.2\alpha_t\) in the first state and \(+0.2\alpha_t\) in the second.
Because \(h(y)=1-y\), the resulting response raises consumption toward \(0.7\) when stock is abundant and lowers it toward \(0.3\) when stock is scarce.
The target itself changes after each decision, so the price response must also keep pace with that change.

\textbf{Stepsize compensation: responding before time runs out.} Now take \(y_t=q_t=0.5\) and reject one request, so \(x_t=0\).
The inventory target for the next arrival and the consumption rate induced by the updated price are
\[
 \begin{aligned}
 q_{t+1}&=\frac{B_t}{n_t-1}
         =0.5+\frac{0.5}{n_t-1},\\
 h(y_{t+1})&=1-\bigl(y_t-\alpha_t q_t\bigr)
         =0.5+\frac{\alpha_t}{2}.
 \end{aligned}
\]
For \((n_t,B_t)=(1000,500)\), the target \(q_{t+1}\) is \(500/999\approx50.05\%\); for \((100,50)\), it is \(50/99\approx50.51\%\).
The same rejection changes the target about ten times as much when one tenth of the opportunities remain.
A stepsize proportional to \(1/(n_t-1)\) makes the change in \(h(y_{t+1})\) respond on this same scale.

For \((n_t,B_t,y_t)=(100,50,0.5)\), the rejection leaves the target at \(q_{t+1}=50/99\).
Conditional on this updated state, the next inventory drift is
\[
 \E_{t+1}[q_{t+2}-q_{t+1}]
 =\frac{q_{t+1}-h(y_{t+1})}{98}.
\]
Applying two stepsizes to the same subgradient \(q_t-x_t=0.5\) gives:
\begin{center}
\small
\setlength{\tabcolsep}{5pt}
\renewcommand{\arraystretch}{1.2}
\begin{tabular}{@{}ccccc@{}}
\toprule
\shortstack{Stepsize\\\(\alpha_t\)} &
\shortstack{Inventory target\\\(q_{t+1}\)} &
\shortstack{Acceptance rate\\\(h(y_{t+1})\)} &
\shortstack{Gap \(h(y_{t+1})-q_{t+1}\)\\(percentage points)} &
\shortstack{Next drift\\in \(q\)}\\
\midrule
\(0.002\) & \(50.505\%\) & \(50.100\%\) & \(-0.405\) & Positive\\
\(2/99\)  & \(50.505\%\) & \(51.010\%\) & \(+0.505\) & Negative\\
\bottomrule
\end{tabular}
\end{center}
After the rejection, the same 50 units must be allocated over 99 rather than 100 arrivals; the smaller step leaves acceptance below the pace required by the remaining inventory, while the larger step lowers the price more quickly and starts catching up with the missed allocation opportunity.

\section{Regret Guarantee and Analysis}\label{sec:analysis}\label{sec:proof}
The following theorem gives the regret guarantee for RASGD.
\begin{theorem}\label{thm:main}
Under Assumption~\ref{ass:gao} and calibration input \eqref{eq:calibration},
Algorithm~\ref{alg:rafo} is prefix feasible and satisfies, for every \(T\ge2\),
\begin{equation}
 0\le\Reg_T(\RA)\le C\log T. \label{eq:main-bound}
\end{equation}
Here \(C\) is independent of \(T\).
\end{theorem}
The guarantee combines logarithmic hindsight loss with exact feasibility and the \(O(m)\) per-arrival computation of Algorithm~\ref{alg:rafo}.
The analysis below explains how price learning and resource adaptation jointly deliver this bound.
Appendix~\ref{app:constants} records the dependence of the constant on the problem parameters.

\subsection{Analysis Framework and Regret Decomposition}\label{sec:roadmap}\label{sec:reduction}
Two sources of loss must be controlled. Inaccurate prices can lead to inferior acceptance decisions, while inventory fluctuations can leave too little usable stock to continue making those decisions.
The analysis compares prices with the fixed reference \(y^*\) and separates these two effects.
Write
\[
 e_t=y_t-y^*,\qquad u_t=q_{t,\B}-d_\B,\qquad
 M=\lfloor T/2\rfloor+1,\qquad \epsT=\frac{\log T}{T}.
\]
Thus \(e_t\) is the price error, \(u_t\) is the binding resource-rate error, and \(M\) separates the early and late phases.
Throughout this section, norms are Euclidean and \(C\) denotes a finite constant independent of \(T\), whose value may change between displays.

We localize the analysis to states where inventory can support the preferred decisions.
During the early phase, the safe region restricts resource rates; during the late phase, it also keeps prices close to \(y^*\).
Let \(\G\) denote successful entrance into the late region at \(M\), with a buffer of nonbinding inventory.
The analytical stopping time \(\tau\) records an early resource exit, an unsuccessful midpoint entrance, a late safe-region exit, or the start of a final window of \(H\) periods.
Here \(H\) is fixed independently of \(T\).
Before \(\tau\), the preferred and implemented actions coincide; the algorithm itself continues throughout the horizon.
Appendix~\ref{app:setup} gives the precise construction and verifies these properties.
The following analysis treats \(T>4H\); the bounded range of smaller horizons is covered at the end of Step~3.

The fixed-reference reward comparison in Lemma~\ref{lem:geometry} separates the loss of a threshold decision into a quadratic price-error term and a linear resource term.
Physical inventory accounting turns the sum of the resource terms into an inventory balance at \(\tau\).
This gives the following reduction, proved in Appendix~\ref{app:setup}.

\begin{lemma}\label{lem:regret-reduction}
For \(T>4H\), the analytical stop above satisfies
\begin{equation}
 \Reg_T(\RA)
 \le L\E\sum_{t<\tau}\norm{e_t}^2+C\E n_\tau,
 \qquad L=\lambda_2m\bar a^2,
 \label{eq:regret-reduction}
\end{equation}
where \(n_\tau=T-\tau+1\).
\end{lemma}

The first term measures cumulative decision error before the stop.
The second bounds the loss charged to the remaining horizon, including the terminal inventory balance.
We split the price-error term at \(M\), and the remaining-horizon term according to whether midpoint entrance succeeds:
\[
 \begin{aligned}
 \E\sum_{t<\tau}\norm{e_t}^2
 &=\underbrace{\E\sum_{\substack{t<M\\t<\tau}}\norm{e_t}^2}_{\text{Step 1}}
   +\underbrace{\E\sum_{M\le t<\tau}\norm{e_t}^2}_{\text{Step 2}},\\
 \E n_\tau
 &=\underbrace{\E[\1_{\G^c}n_\tau]}_{\text{Step 1}}
   +\underbrace{\E[\1_\G n_\tau]}_{\text{Step 3}}.
 \end{aligned}
\]
It suffices to bound each term on the right by \(C\log T\).
The three steps establish these bounds in the following order:
\begin{itemize}
\item \emph{Step~1: Early learning and midpoint entrance.}
Bound early price error and the remaining-horizon charge on \(\G^c\); establish the midpoint accuracy and inventory buffer needed on \(\G\).
\item \emph{Step~2: Late stabilization through matched updates.}
Starting from these midpoint estimates, control late price error through matched updates. This completes the cumulative price-error bound, leaving only the remaining horizon on \(\G\).
\item \emph{Step~3: Controlling premature exits.}
Use the late estimates to bound first-exit probabilities and hence the remaining-horizon charge on \(\G\). Substitute the completed bounds into \eqref{eq:regret-reduction} to conclude Theorem~\ref{thm:main}.
\end{itemize}

\subsection{Step 1: Early Learning and Midpoint Entrance}\label{sec:early}
We first control the early part of the price-error sum and the cost of unsuccessful midpoint entrance, while obtaining the initial conditions for late stabilization.
Price learning and inventory adjustment operate on different scales:
\[
 \alpha_t\asymp\frac1{t+s_0},\qquad
 \beta_t:=\frac1{n_t-1}\le\frac2T,\qquad t<M.
\]
Inventory rates therefore change slowly while prices are learned, but both updates use the same requests.
A price estimate alone would not guarantee adequate inventory at the midpoint.
The joint energy and localization argument in Appendix~\ref{app:early} controls both; here \(S\) is the early resource stop from Appendix~\ref{app:setup}.

\begin{lemma}\label{lem:early}
For \(T>4H\) and \(1\le t\le M\),
\begin{align}
 \E[\1\{t<S\}\norm{e_t}^2]
 &\le C\left(\frac1{t+s_0}+\epsT\right), \label{eq:early-pointwise}\\
 \Prob(\G^c)&\le C\epsT, \label{eq:G-prob}\\
 \E[\1_\G(\norm{e_M}^2+\norm{u_M}^2)]
 &\le C\epsT. \label{eq:G-moment}
\end{align}
\end{lemma}

For \(t<M\), \(\{t<\tau\}=\{t<S\}\).
Summing \eqref{eq:early-pointwise} therefore controls the early contribution:
\begin{equation}
 \E\sum_{\substack{t<M\\t<\tau}}\norm{e_t}^2
 \le C\sum_{t=1}^{M-1}\left(\frac1{t+s_0}+\epsT\right)
 \le C\bigl(\log T+T\epsT\bigr)\le C\log T.
 \label{eq:early-cumulative}
\end{equation}
The reciprocal term captures learning error, while \(\epsT\) allows for the coupled inventory fluctuations; both accumulate only logarithmically.
On successful entrance, \eqref{eq:G-moment} initializes the price--inventory combination needed for Step~2:
\[
 w:=e_{M,\B}+\kappa u_M,\qquad
 \E[\1_\G\norm w^2]
 \le C\E[\1_\G(\norm{e_M}^2+\norm{u_M}^2)]
 \le C\epsT.
\]
These are weighted moments on \(\G\), not an assumption that entrance always succeeds.
The remaining horizon on unsuccessful paths is already controlled by \eqref{eq:G-prob}:
\begin{equation}
 \E[\1_{\G^c}n_\tau]
 \le T\Prob(\G^c)\le CT\epsT\le C\log T.
 \label{eq:failed-entrance-charge}
\end{equation}
Combining \eqref{eq:early-cumulative} and \eqref{eq:failed-entrance-charge} with \eqref{eq:regret-reduction} leaves
\begin{equation}
 \Reg_T(\RA)
 \le C\log T
   +L\E\sum_{M\le t<\tau}\norm{e_t}^2
   +C\E[\1_\G n_\tau].
 \label{eq:after-early}
\end{equation}
Thus early loss and unsuccessful entrance are settled.
The midpoint moments and nonbinding buffer on \(\G\) supply the inputs for Step~2; the two unresolved terms in \eqref{eq:after-early} are handled by Steps~2 and~3, respectively.

\subsection{Step 2: Late Stabilization through Matched Updates}\label{sec:stability}
Using the midpoint estimates from Step~1, we now bound the late price-error sum in \eqref{eq:after-early}.
This requires control only before \(\tau\), without yet assuming that late exits are rare.
As the remaining horizon shrinks, each consumption discrepancy changes the resource rate more strongly.
The late stepsize matches this change: \(\alpha_t=\kappa\beta_t\).
On \(\G\), the binding cap, binding projection, and feasibility filter are inactive before \(\tau\), so the coupled updates \eqref{eq:feedback-recursions} give
\begin{equation}
 \begin{aligned}
 e_{t+1,\B}-e_{t,\B}&=-\kappa(u_{t+1}-u_t),&&M\le t<\tau,\\
 e_{t,\B}+\kappa u_t&=w,&&M\le t\le\tau.
 \end{aligned}
 \label{eq:invariant}
\end{equation}
The shared consumption noise cancels in this combination.
Since \(e_{t,\B}=w-\kappa u_t\), surplus inventory lowers its binding price relative to the fixed offset \(w\), while scarcity raises it.
Fixed-reference coercivity and the calibrated gain make this feedback restoring (Lemma~\ref{lem:restoring}):
\[
 \begin{aligned}
 \E_t[u_{t+1}-u_t]&=\beta_t\bigl(q_{t,\B}-h_\B(y_t)\bigr),\\
 \ip{u_t}{q_{t,\B}-h_\B(y_t)}
 &\le-\frac18\norm{u_t}^2
       +C\norm w^2+C\norm{y_{t,\N}}^2.
 \end{aligned}
\]
The negative term stabilizes binding inventory, up to the midpoint error and the nonbinding prices.
For nonbinding resources, strict slack pushes prices toward zero; Lemma~\ref{lem:reflection} controls their reflected noise.
Together with Step~1, this yields the bound below: \(\epsT\) is inherited from the midpoint, while \(1/n_t\) accounts for accumulated late noise (proof in Appendix~\ref{app:late}).

\begin{lemma}\label{lem:late}
For \(T>4H\) and \(M\le t\le T-H+1\),
\begin{equation}
 \E[\1_\G\1\{t<\tau\}(\norm{e_t}^2+\norm{u_t}^2)]
 \le C\left(\epsT+\frac1{n_t}\right).
 \label{eq:late-bounds}
\end{equation}
\end{lemma}

Since \(\tau\le M\) on \(\G^c\), only successful entrances contribute to the late error.
Summing \eqref{eq:late-bounds} gives
\begin{equation}
 \E\sum_{M\le t<\tau}\norm{e_t}^2
 \le C\sum_{n=H}^{n_M}\left(\epsT+\frac1n\right)
 \le C\left(T\epsT+\log T\right)
 \le C\log T.
 \label{eq:late-cumulative}
\end{equation}
Combining this with the early bound \eqref{eq:early-cumulative} completes the first term of the regret reduction:
\begin{equation}
 \E\sum_{t<\tau}\norm{e_t}^2
 =
 \E\sum_{\substack{t<M\\t<\tau}}\norm{e_t}^2
 +\E\sum_{M\le t<\tau}\norm{e_t}^2
 \le C\log T.
 \label{eq:cumulative}
\end{equation}
In particular, \eqref{eq:late-cumulative} resolves the first outstanding term in \eqref{eq:after-early}, reducing the regret bound to
\begin{equation}
 \Reg_T(\RA)\le C\log T+C\E[\1_\G n_\tau].
 \label{eq:after-late}
\end{equation}
Only the remaining horizon after successful entrance is still unbounded.

\subsection{Step 3: Controlling Premature Exits}\label{sec:implementation}\label{sec:maximal}
By \eqref{eq:after-late}, the sole remaining task is to bound \(\E[\1_\G n_\tau]\).
The stopped moment bound \eqref{eq:late-bounds} discards a path at its exit, so it does not directly bound this quantity.
We use it instead to control maxima that include the first exit state, and weight each exit by the number of periods left.
We group possible exits by their remaining-horizon scale, halving that scale until at most \(2H\) periods remain:
\[
 \begin{gathered}
 N_0=n_M,\qquad N_{j+1}=\lfloor N_j/2\rfloor,\qquad
 J=\min\{j:N_j\le2H\},\\
 (a_j,b_j]=(T-N_j+1,\,T-N_{j+1}+1],\qquad 0\le j<J.
 \end{gathered}
\]
On \(\G\), these blocks cover every exit with \(n_\tau>2H\), and satisfy
\[
 n_t\le N_j\quad(a_j<t\le b_j),\qquad
 \sum_{j<J}N_j\le T+1,\qquad J=O(\log T).
\]
On each block, Step~2 controls the entrance error and the drift, while the accumulated noise variance is \(O(1/N_j)\).
Maximum estimates that retain the exit-producing update turn these bounds into the following lemma, proved in Appendix~\ref{app:exit}.
Nonbinding lower exits are controlled there using the midpoint buffer and positive resource drift.

\begin{lemma}\label{lem:remaining}
For \(T>4H\), the block exit probabilities satisfy
\begin{equation}
 p_j:=\Prob\big(\G,\ a_j<\tau\le b_j,\ n_\tau>2H\big)
 \le C\left(\epsT+\frac1{N_j}\right),
 \qquad 0\le j<J.
 \label{eq:block-exit}
\end{equation}
\end{lemma}

An exit in block \(j\) costs at most \(N_j\) periods, so the \(1/N_j\) probability term contributes only \(O(1)\) per block.
On \(\G\), exits outside these blocks leave at most \(2H\) periods. Summing the charges gives
\begin{align}
 \E[\1_\G n_\tau]
 &\le2H+\sum_{j<J}N_jp_j \notag\\
 &\le2H+C\epsT\sum_{j<J}N_j+CJ \notag\\
 &\le C\bigl(T\epsT+H+J\bigr)\le C\log T.
 \label{eq:successful-entrance-charge}
\end{align}
This resolves the last term in \eqref{eq:after-late}.
Together with the failed-entrance bound \eqref{eq:failed-entrance-charge} from Step~1, it also gives the full remaining-horizon estimate:
\begin{equation}
 \E n_\tau
 =\E[\1_{\G^c}n_\tau]+\E[\1_\G n_\tau]
 \le C\log T.
 \label{eq:remaining}
\end{equation}

To complete Theorem~\ref{thm:main}, substitute \eqref{eq:cumulative} and \eqref{eq:remaining} into \eqref{eq:regret-reduction}:
\[
 \Reg_T(\RA)
 \le
 \underbrace{L\E\sum_{t<\tau}\norm{e_t}^2}_{O(\log T)\text{ by Steps 1--2}}
 +
 \underbrace{C\E n_\tau}_{O(\log T)\text{ by Step 3}}
 \le C\log T.
\]
The feasibility filter enforces \(a_tx_t\le B_t\) on every path, giving prefix feasibility and nonnegative regret.
The bounded range of smaller horizons is covered by
\[
 0\le\Reg_T(\RA)\le2\bar rT
 \le8\bar rH
 \le\frac{8\bar rH}{\log 2}\log T,
 \qquad 2\le T\le4H.
\]
The large- and small-horizon bounds, together with pathwise feasibility, complete the proof of Theorem~\ref{thm:main}.

\subsection{Optimality of the Logarithmic Rate}\label{sec:lower-bound}
 The following embedding of the multisecretary lower bound of \citet{Bray2025} establishes the optimality of the horizon order within Assumption~\ref{ass:gao}, including a mixed active set.

\begin{corollary}\label{cor:lower-bound}
Consider the two-resource packing instance
\begin{equation}
 r\sim\operatorname{Unif}[0,1],\qquad
 a=(1,Z)^\top,\quad Z\sim\operatorname{Bernoulli}(1/2),\qquad
 d=(1/2,1)^\top, \label{eq:lower-bound-instance}
\end{equation}
where $r$ and $Z$ are independent and requests are i.i.d.
This instance satisfies Assumption~\ref{ass:gao}, with $\B=\{1\}$ and $\N=\{2\}$.
Let $\Pi_T$ be the class of all possibly randomized, nonanticipating binary policies satisfying \eqref{eq:inventory}, including policies that know the arrival distribution.
There are constants $\gamma>0$ and $T_0<\infty$, independent of the policy and horizon, such that for every even $T\ge T_0$,
\begin{equation}
 \inf_{\pi\in\Pi_T}\Reg_T(\pi)\ge\gamma\log T.
 \label{eq:lower-bound}
\end{equation}
\end{corollary}

Appendix~\ref{app:lower-bound} verifies the assumptions and reduces the instance to the known multisecretary lower bound.
The embedding preserves the fractional hindsight benchmark and allows the policy to know the distribution.
Thus the logarithmic order is unavoidable over the stated class, although particular instances may have smaller regret.
Different mechanisms yield polynomial lower bounds with thin size-weighted value-to-size mass \citep{ZhangContinuous2026} and \(\Omega(\log^2 T)\) with support gaps \citep{ZhangBellman2026}.

\section{Numerical Experiments}
\label{sec:experiments}

We compare \algname{RASGD} with six unknown-distribution policies across resource levels and horizons. The experiments assess hindsight regret, online runtime, and the contributions of \algname{RASGD}'s resource feedback and stepsize design.

\subsection{Benchmarks and Evaluation}
\label{sec:exp-design}
Requests are i.i.d., rewards and consumptions are nonnegative, and online actions are binary. We use two benchmarks:
\begin{enumerate}[label=\textbf{B\arabic*.},leftmargin=*]
\item \textbf{Single resource.} The unit-consumption multisecretary model has $a_t=1$, $r_t\sim U[0,1]$, and capacity $B_1=\rho T$.
\item \textbf{Multiple resources.} There are $m=10$ resources unless stated otherwise, with mutually independent $r_t,a_{t,1},\ldots,a_{t,m}\sim U[0,2]$. Capacities are $B_1=Td$, with either homogeneous or mixed coverage as defined below.
\end{enumerate}

The coverage label $\rho_i=B_{1,i}/(T\E[a_i])$ measures capacity relative to expected offered demand. A homogeneous B2 instance gives all ten resources the same coverage $\rho$. In the separate \emph{B2 mixed} instance, resources 1--5 have coverage $0.50$ and resources 6--10 have coverage $0.90$, producing binding and slack resources in the same problem. Thus, mixed is a capacity configuration, not an additional scalar coverage level.

Smaller values indicate greater scarcity; coverage above one exceeds mean demand but need not cover every realized B2 stream. Resource sweeps use ten levels from $0.10$ to $1.25$ at $T=1{,}000,2{,}000,5{,}000$. Horizon sweeps fix coverage at $0.50$ or $0.90$ and vary $T$ up to $20{,}000$. In B2, the label \emph{all binding} refers to the homogeneous $0.50$ setting.

Every method is evaluated against the same realized fractional hindsight optimum:
\[
 R_A=\max_{0\le z\le1}\left\{\sum_t r_tz_t:\sum_ta_tz_t\le Td\right\}
       -\sum_t r_tx_t^A.
\]
Lower regret is better. The optimum is computed by exact sorting for B1 and a certified LP for B2. Each quality configuration uses 30 independent streams shared across methods; reported intervals are 95\% Monte Carlo intervals. All deployed policies apply the same inventory feasibility check.

\subsection{Algorithms and Calibration}
\label{sec:exp-protocol}
The seven methods receive the same arrival information and do not know its distribution. We present the five first-order methods first, followed by the two empirical-LP policies.
\begin{enumerate}[label=\textbf{\arabic*.},leftmargin=*]
\item \textbf{\algname{RASGD} (ours).} Algorithm~\ref{alg:rafo} updates prices by projected stochastic gradient descent at every arrival, using remaining inventory and a stepsize that decreases early and increases near the horizon. It solves no empirical LP.
\item \textbf{\algname{Gao A3}.} Algorithm~3 of \citet{Gao2026} maintains separate learning and decision price paths during exploration, then transfers the learner to the decision path. Our feasibility check leaves its virtual updates unchanged.
\item \textbf{\algname{Ma A5}.} Algorithm~5 of \citet{Ma2025} performs a first-order price update at every arrival with box projection and refreshes the resource target at epoch boundaries.
\item \textbf{\algname{LSY A1}.} Algorithm~1 of \citet{LiSunYe2020} uses a fixed resource target and a projected price update. With the same constant stepsize, initialization, and feasibility check, it coincides with the Euclidean version of Algorithm~1 of \citet{BalseiroLuMirrokni2020}; we report this shared version once.
\item \textbf{\algname{DMD A1}.} Algorithm~1 of \citet{BalseiroLuMirrokni2020} uses the entropic mirror map, giving multiplicative price updates with a fixed resource target.
\item \textbf{\algname{Li--Ye A2}.} Algorithm~2 of \citet{LiYe2022} re-solves an empirical LP at geometric checkpoints, retaining the original resource target.
\item \textbf{\algname{Li--Ye A3}.} Algorithm~3 of \citet{LiYe2022} re-solves an empirical LP after every arrival using the observed history and remaining resource rate. The same empirical pricing rule appears in Algorithm~3 of \citet{Bray2025}, after time-index alignment.
\end{enumerate}
Executable adaptations and implementation conventions are documented in the accompanying reproducibility materials.

Practical first-order parameters are selected on validation streams with ten seeds disjoint from the test seeds at $T=5{,}000$, using mean regret across resource settings. \algname{RASGD}, \algname{Gao A3}, and \algname{Ma A5} start with 12 candidates each; \algname{LSY A1} and \algname{DMD A1} each search six stepsize gains. For every tunable gain, selection at an initial search boundary triggers one extension, dividing the minimum or multiplying the maximum by four while retaining the other initial coordinates. One parameter vector per benchmark is then frozen across coverage, horizon, and dimension tests. The search grids, validation scores, and final parameters are retained in the reproducibility materials. Practical parameter selection is independent of the theoretical \algname{RASGD} calibration $\kappa=2/\mu_0$.

\subsection{Regret Across Resource Levels and Horizons}
\label{sec:exp-quality}

Table~\ref{tab:exp-main-regret} and Figure~\ref{fig:exp-resource} show that \algname{RASGD} achieves regret close to that of \algname{Li--Ye A3} in the representative settings, using only a projected gradient update per arrival. The gap is particularly small at moderate-to-high coverage. \algname{RASGD} even attains a slightly lower mean regret on B1 at coverage $0.90$: $0.57$ versus $0.80$ for \algname{Li--Ye A3}. Although \algname{Li--Ye A3} retains an advantage in most B2 settings, this comparison highlights \algname{RASGD}'s strong allocation quality without empirical LP re-solving.

At $T=5{,}000$, \algname{RASGD} also has lower mean regret than the four other first-order methods in both benchmarks below mean-demand coverage and in B2 mixed. At abundant B2 coverage, \algname{Ma A5} and \algname{LSY A1} can slightly outperform \algname{RASGD}. Entropic \algname{DMD A1} improves over \algname{LSY A1} in several scarce or mixed settings, but its positive initial prices incur additional regret when resources are abundant.

\begin{table}[htbp]
\centering\SingleSpacedXI\footnotesize\setlength{\tabcolsep}{2pt}
\caption{Hindsight regret at selected resource-coverage levels at $T=5{,}000$. Entries are means $\pm$ 95\% confidence half-widths over 30 paired streams; lower is better.}\label{tab:exp-main-regret}
\begin{tabular}{lrrrrrrr}\toprule Coverage & \algname{RASGD} & \algname{Gao A3} & \algname{Ma A5} & \algname{LSY A1} & \algname{DMD A1} & \algname{Li--Ye A2} & \algname{Li--Ye A3}\\\midrule
\multicolumn{8}{l}{\textbf{B1: single resource}}\\[2pt]
$0.50$ & $2.87\pm0.31$ & $5.32\pm0.74$ & $5.87\pm0.51$ & $17.60\pm0.76$ & $10.00\pm0.59$ & $13.39\pm3.57$ & $1.74\pm0.28$\\
$0.75$ & $1.45\pm0.23$ & $3.74\pm0.64$ & $3.80\pm0.29$ & $11.28\pm0.52$ & $5.19\pm0.32$ & $9.30\pm2.01$ & $1.28\pm0.19$\\
$0.90$ & $0.57\pm0.14$ & $2.03\pm0.49$ & $1.69\pm0.23$ & $4.89\pm0.43$ & $4.12\pm0.13$ & $6.04\pm1.79$ & $0.80\pm0.17$\\
$1.25$ & $0.00\pm0.00$ & $0.00\pm0.00$ & $0.00\pm0.00$ & $0.00\pm0.00$ & $1.39\pm0.10$ & $0.95\pm0.14$ & $0.00\pm0.00$\\
\midrule
\multicolumn{8}{l}{\textbf{B2: multiple resources, homogeneous coverage}}\\[2pt]
$0.50$ & $38.40\pm1.88$ & $53.32\pm3.26$ & $70.96\pm4.23$ & $59.66\pm2.81$ & $55.89\pm1.48$ & $94.70\pm10.44$ & $31.69\pm1.71$\\
$0.75$ & $26.14\pm1.63$ & $39.29\pm1.68$ & $40.35\pm2.66$ & $36.33\pm3.53$ & $43.30\pm1.50$ & $62.06\pm10.01$ & $24.15\pm1.58$\\
$0.90$ & $13.13\pm0.77$ & $37.93\pm0.91$ & $23.22\pm1.75$ & $19.85\pm0.68$ & $33.35\pm0.85$ & $20.81\pm2.78$ & $12.11\pm1.05$\\
$1.25$ & $0.44\pm0.06$ & $3.19\pm0.22$ & $0.37\pm0.08$ & $0.30\pm0.03$ & $19.41\pm0.78$ & $2.41\pm0.42$ & $0.43\pm0.23$\\
\midrule
\multicolumn{8}{l}{\textbf{B2: multiple resources, mixed coverage}}\\[2pt]
$0.50\,/\,0.90$ & $27.43\pm2.30$ & $51.77\pm3.49$ & $65.53\pm7.18$ & $68.48\pm3.17$ & $42.79\pm1.41$ & $84.72\pm14.31$ & $21.74\pm1.81$\\
\bottomrule\end{tabular}
\par\smallskip\raggedright\footnotesize Coverage is initial inventory divided by expected total demand. B2 has ten resources. Homogeneous rows give every resource the displayed coverage. The mixed row assigns coverage $0.50$ to resources 1--5 and $0.90$ to resources 6--10; it is a separate capacity configuration.
\end{table}

\begin{figure}[htbp]
\centering\small
\centering\includegraphics[width=.88\textwidth]{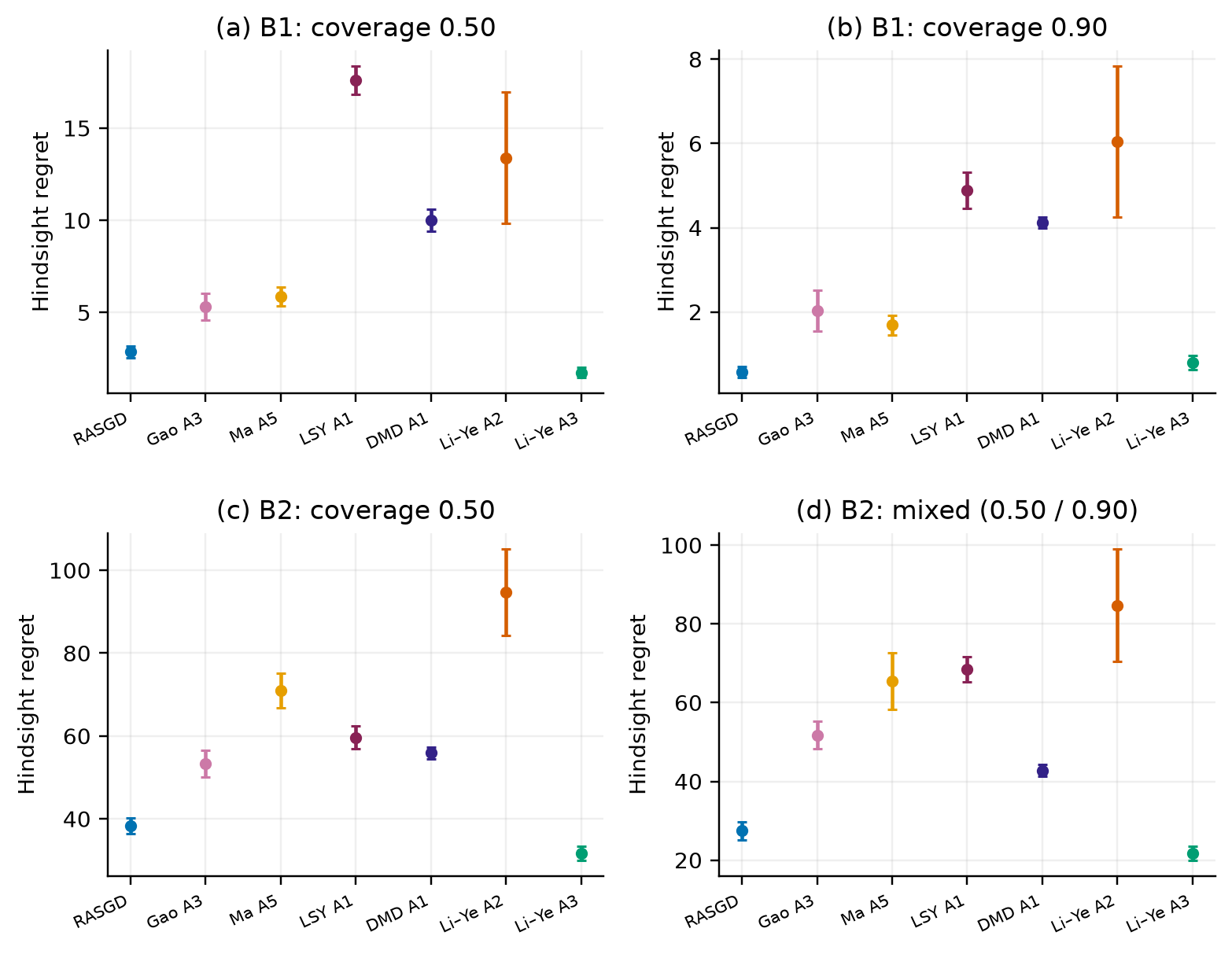}
\caption{Hindsight regret at $T=5{,}000$. Top row: B1 at coverage $0.50$ and $0.90$. Bottom row: B2 at coverage $0.50$ and B2 mixed (five resources at $0.50$ and five at $0.90$). Points are means and bars are 95\% intervals over 30 paired streams.}
\label{fig:exp-resource}
\end{figure}

The horizon comparison in Figure~\ref{fig:exp-horizon} supports the same conclusion through $T=20{,}000$. \algname{RASGD} remains the best-performing first-order method in the displayed settings. At coverage $0.90$, its regret remains close to \algname{Li--Ye A3}'s, with a slightly lower B1 mean at the largest horizon; \algname{Li--Ye A3} retains a clearer advantage at coverage $0.50$. Together with the runtime results in Section~\ref{sec:exp-time}, these findings show that \algname{RASGD} combines allocation quality approaching per-arrival re-solving with first-order computational efficiency. These finite-horizon comparisons do not estimate asymptotic regret rates.

\begin{figure}[htbp]
\centering\small
\centering\includegraphics[width=.86\textwidth]{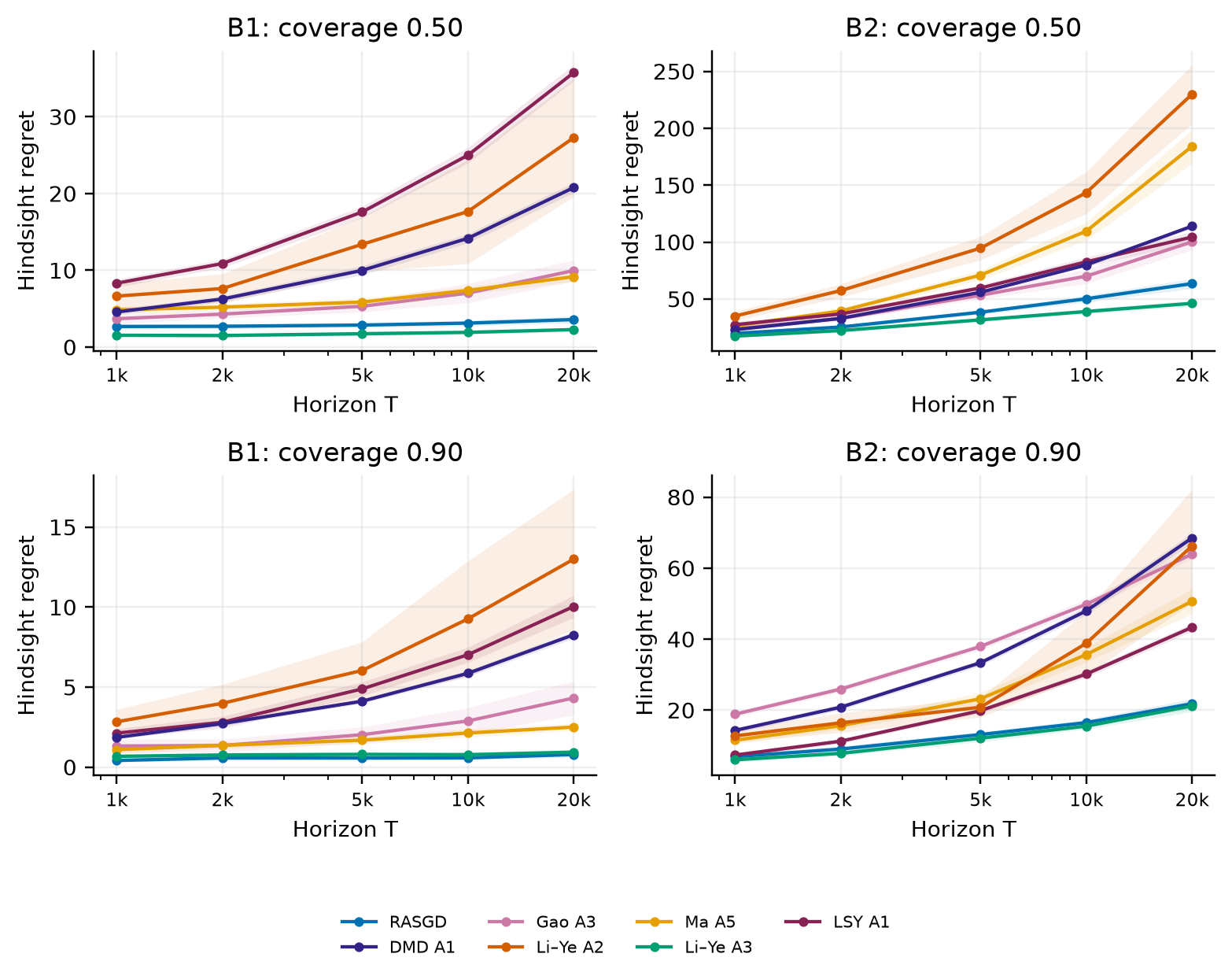}
\caption{Hindsight regret under frozen parameters. Rows use coverage $0.50$ and $0.90$; columns show B1 and B2. Bands are 95\% intervals over 30 paired streams. The horizon axis is logarithmic and the regret axis is linear.}
\label{fig:exp-horizon}
\end{figure}

\subsection{Online Runtime}
\label{sec:exp-time}

Table~\ref{tab:exp-timing-5000} and Figure~\ref{fig:exp-timing-5000} report online time for $T=5{,}000$. All methods were measured on the same Apple M5 MacBook Air with 10 CPU cores and 32 GB memory, running macOS 26.4, with execution restricted to one thread. Each first-order method uses a specialized compiled kernel. 
 Timing includes initialization and all online operations, including empirical LP solves, and excludes compilation, input generation, hindsight optimization, and logging. Ten independent streams each contribute the median of three interleaved timing blocks; software versions and measurement details are recorded in the reproducibility materials.

\begin{table}[htbp]
\centering\SingleSpacedXI\small\setlength{\tabcolsep}{4pt}
\caption{Online time in milliseconds per stream on B2 at $T=5{,}000$ and coverage $0.50$. Entries are means $\pm$ 95\% confidence half-widths over ten streams.}\label{tab:exp-timing-5000}
\begin{tabular}{lrr}\toprule Method & $m=10$ & $m=50$\\\midrule
\algname{RASGD} & $0.262\pm0.039$ & $0.797\pm0.113$\\
\algname{Gao A3} & $0.324\pm0.092$ & $0.979\pm0.043$\\
\algname{Ma A5} & $0.189\pm0.033$ & $0.760\pm0.046$\\
\algname{LSY A1} & $0.253\pm0.138$ & $0.694\pm0.029$\\
\algname{DMD A1} & $0.717\pm0.221$ & $2.638\pm0.100$\\
\algname{Li--Ye A2} & $62.828\pm23.299$ & ---\\
\algname{Li--Ye A3} & $18{,}184.344\pm5{,}307.279$ & ---\\
\bottomrule\end{tabular}
\par\smallskip\raggedright\footnotesize ---: omitted due to excessive runtime. Each reported stream contributes the median of three timing blocks.
\end{table}

\algname{RASGD}, \algname{Gao A3}, \algname{Ma A5}, and \algname{LSY A1} have comparable sub-millisecond mean times at both dimensions. Entropic \algname{DMD A1} is slower, taking about $2.7$ and $3.3$ times \algname{RASGD}'s time at $m=10$ and $m=50$, respectively; its update evaluates an exponential for every resource at every arrival. Both empirical-LP methods take substantially longer at $m=10$, especially \algname{Li--Ye A3} with its per-arrival solves. These are implementation-specific wall-clock comparisons, rather than a statistical equivalence claim. Together with the regret results, they support \algname{RASGD}'s favorable allocation quality at a computational cost close to the other additive first-order updates.

\begin{figure}[htbp]
\centering\small
\centering\includegraphics[width=.95\textwidth]{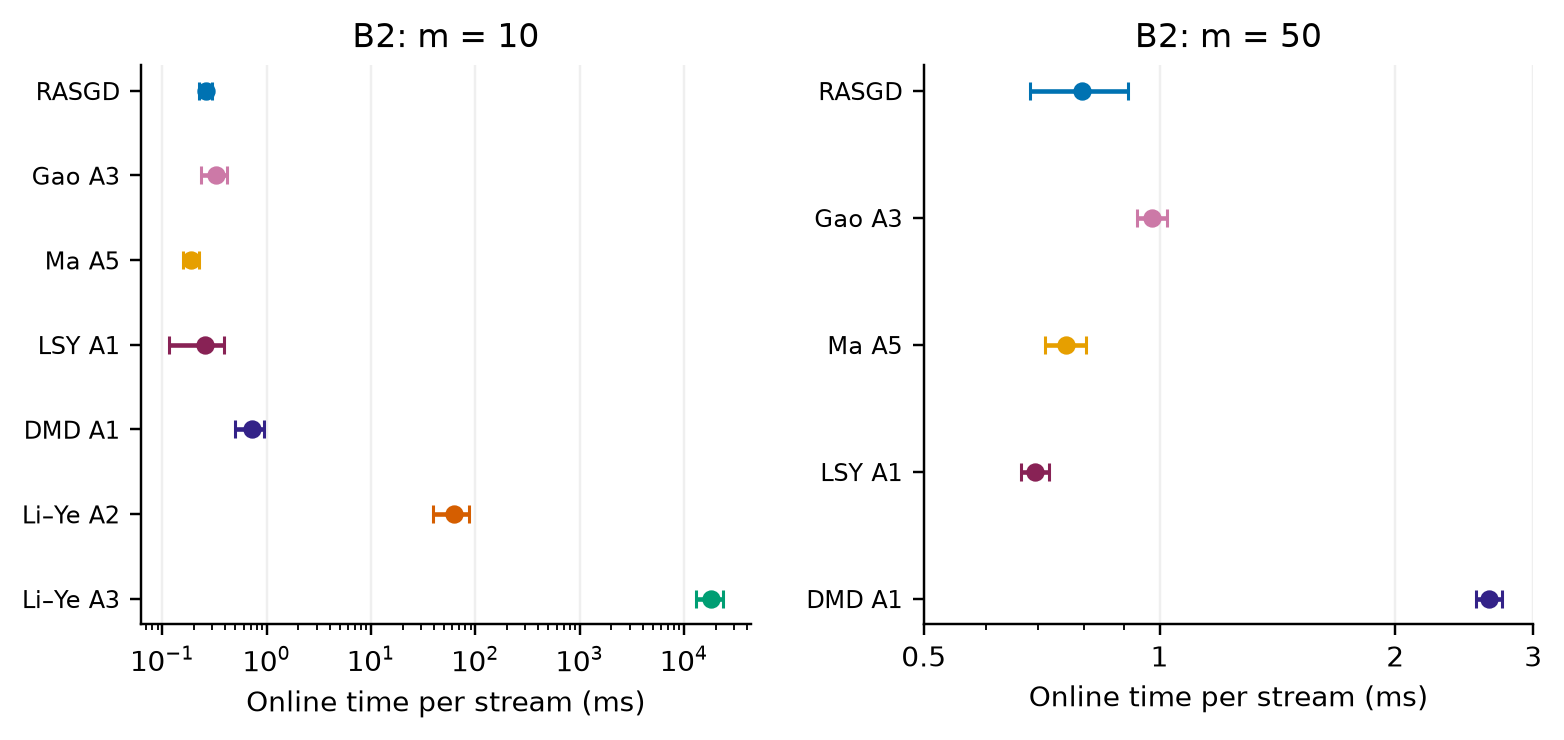}
\caption{Single-thread online time on B2 at $T=5{,}000$ and coverage $0.50$. Points and bars show means and 95\% intervals over ten streams. Both panels use logarithmic axes with panel-specific ranges. The $m=50$ comparison includes the five first-order methods under the pilot cost rule.}
\label{fig:exp-timing-5000}
\end{figure}

\subsection{Ablation of Resource Feedback and Stepsize}
\label{sec:exp-mechanisms}

We test the contributions of two \algname{RASGD} components: resource feedback and the late-horizon stepsize increase. Removing feedback replaces the remaining-inventory target by the fixed initial target $d$; removing the late-horizon increase makes the stepsize decrease throughout the horizon, subject to the same truncation. Together with full \algname{RASGD} and the variant removing both components, these give four policies. We evaluate them on B2 mixed, with five resources at coverage $0.50$ and five at $0.90$, at $T=2{,}000$ and $10{,}000$, using 30 paired streams.

Table~\ref{tab:exp-ablation} and Figure~\ref{fig:exp-ablation} compare two calibration protocols. The shared-parameter protocol fixes $(\kappa,s_0)=(4,64)$ for all four policies, isolating the effect of changing components. The independently tuned protocol lets each policy select its own parameters on separate validation streams with the same search grid and boundary-extension rule, testing whether retuning compensates for a removed component.

\begin{table}[htbp]
\centering\SingleSpacedXI\small\setlength{\tabcolsep}{4pt}
\caption{Component ablation on B2 mixed. Entries are mean hindsight regrets over 30 paired streams; lower is better.}\label{tab:exp-ablation}
\begin{tabular*}{\textwidth}{@{\extracolsep{\fill}}lrrrr@{}}\toprule
& \multicolumn{2}{c}{Shared parameters} & \multicolumn{2}{c}{Independently tuned}\\
\cmidrule(lr){2-3}\cmidrule(l){4-5}
Variant & $T=2{,}000$ & $T=10{,}000$ & $T=2{,}000$ & $T=10{,}000$\\\midrule
\textbf{Full \algname{RASGD}} & \textbf{20.51} & \textbf{34.10} & \textbf{20.51} & \textbf{34.10}\\
No resource feedback & 31.32 & 73.63 & 32.21 & 78.69\\
No late-horizon increase & 27.83 & 62.46 & 32.39 & 63.88\\
Neither component & 32.49 & 74.85 & 34.82 & 68.67\\
\bottomrule\end{tabular*}
\par\smallskip\raggedright\footnotesize Shared parameters use $(\kappa,s_0)=(4,64)$ for every variant. Independently tuned parameters are selected on separate validation streams. Figure~\ref{fig:exp-ablation} reports the corresponding 95\% intervals.
\end{table}

\begin{figure}[htbp]
\centering\small
\centering\includegraphics[width=.95\textwidth]{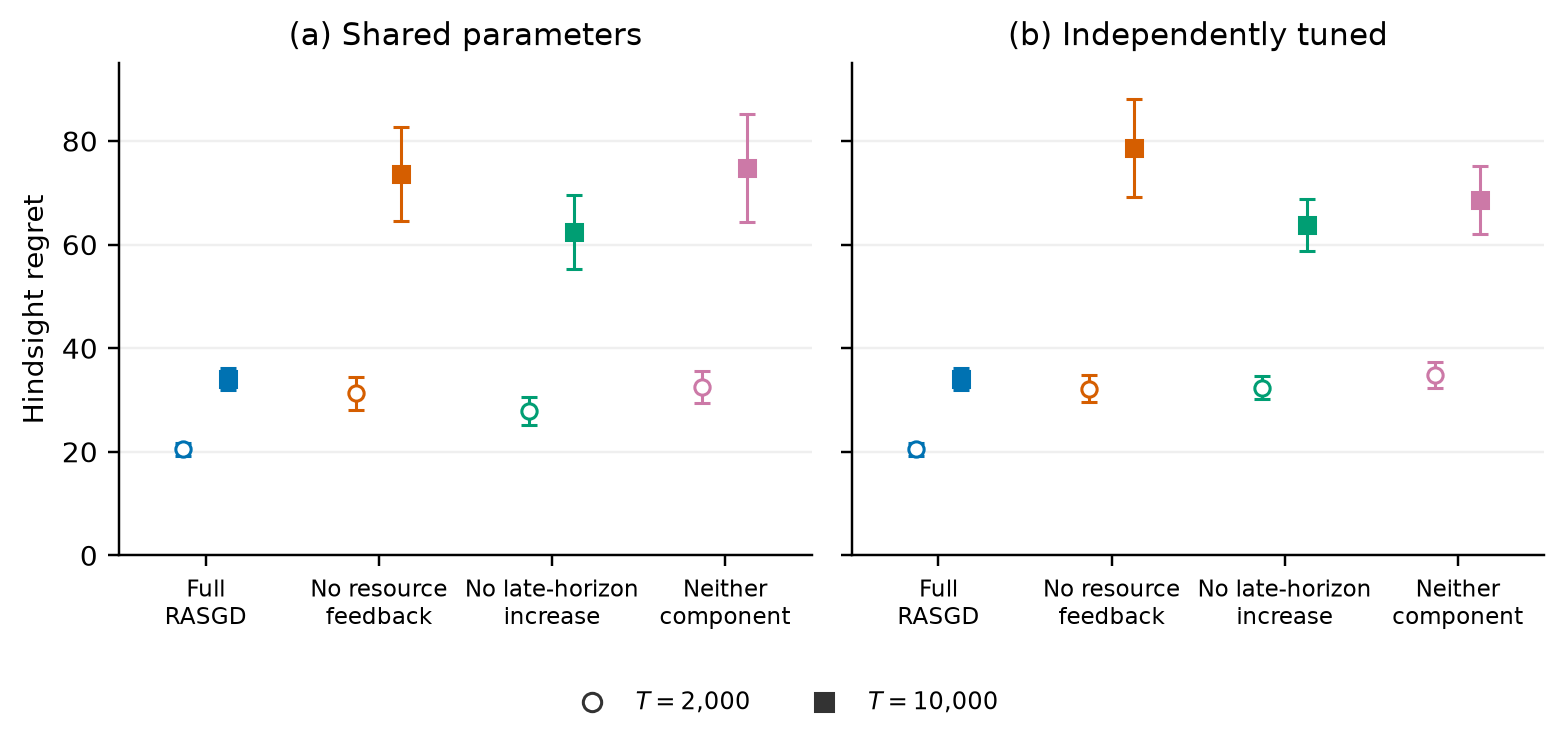}
\caption{Component ablation on B2 mixed: shared parameters (left) and independent tuning (right). Points and bars show mean regret and 95\% intervals over 30 paired streams. Open circles indicate $T=2{,}000$; filled squares indicate $T=10{,}000$. Both panels use the same vertical scale.}
\label{fig:exp-ablation}
\end{figure}

Removing either component increases mean regret at both horizons under both protocols. Full \algname{RASGD} remains best after independent retuning, supporting the contribution of both resource feedback and the late-horizon stepsize increase in this mixed-resource setting. Retuning need not lower test regret because selection uses independent validation data.

\section{Conclusion}\label{sec:conclusion}
This paper develops RASGD, a resource-adaptive stochastic gradient
descent framework for online linear programming with optimal regret and without resolving.      
RASGD couples per-arrival resource feedback with a two-sided stochastic-gradient stepsize.
Under standard non-degeneracy conditions, it achieves optimal \(O(\log T)\) fractional-hindsight regret with prefix feasibility and \(O(m)\) work and memory per arrival.
The analysis follows the two time scales built into the algorithm.
Early on, a joint price--inventory energy shows that prices learn the nominal optimum while remaining-inventory rates move slowly; later, the matched stepsize yields a price--inventory invariant that turns the nominal threshold response into restoring feedback for binding resources.
Inclusive first-exit bounds then show that rare departures from this stable region contribute only logarithmically, avoiding any need to track optimizers at changing capacities.

The numerical results show that RASGD attains regret competitive with
per-arrival LP re-solving and lower mean regret than the other tested
first-order methods in representative resource-constrained settings.
RASGD retains the computational efficiency of first-order updates and is
substantially faster than per-arrival LP re-solving in the measured
implementations. Ablation experiments on the mixed-resource benchmark
support the contributions of resource feedback and the late-horizon
stepsize increase.

These results highlight RASGD's potential as a fast and efficient OLP
algorithm for large-scale AI services, including LLM admission control
under token and compute budgets.
More broadly, jointly designing resource feedback and stepsizes to match
resource dynamics offers a promising approach to other resource-constrained
learning and decision problems.
Extending this principle to stochastic replenishment, time-varying demand,
and joint admission and scheduling in AI services presents natural
directions for future work.

\bibliographystyle{informs2014}
\bibliography{references}
\clearpage
\begin{APPENDICES}
\section{Proofs for the Regret Analysis}\label{app:analysis}\label{app:auxiliary}
The proofs follow the framework and the three steps of Section~\ref{sec:analysis}.
Each subsection collects the supporting estimates and complete proofs for its corresponding part of the main text.
Unless stated otherwise, constants are independent of \(T\), and we consider \(T>4H\).

\subsection{Proofs for the Analysis Framework and Regret Decomposition}\label{app:setup}\label{app:geometry}
For \(e=y-y^*\), define \(\Psi(e)=F(y)-F(y^*)=h(y^*)-h(y)\).
We first record the geometry at the fixed reference price.

\begin{lemma}\label{lem:geometry}
Under Assumption~\ref{ass:gao}, \(y^*\) is unique.
For \(e=y-y^*\), \(y\in\Y\), and \(R_*(y)=\E[rX(y;r,a)]\), one may take \(L=\lambda_2m\bar a^2\) so that
\begin{align}
 \ip e{\Psi(e)}&\ge\mulo\norm e^2,&
 \norm{\Psi(e)}&\le L\norm e,\label{eq:geometry}\\
 0&\le R_*(y^*)-R_*(y)-y^{*\top}[d-h(y)]
 \le L\norm e^2.\label{eq:quadratic}
\end{align}
Moreover, \(R_*(y^*)=f_d(y^*)\).
\end{lemma}
\begin{proof}[Proof of Lemma~\ref{lem:geometry}]
Condition (\ref{ass:complementarity}) gives \eqref{eq:reference-slack}, so \(F(y^*)\ge0\) and \(y^{*\top}F(y^*)=0\).

Write \(S_y(a)=\Prob(r\ge a^\top y\mid a)\).
Survival monotonicity and \eqref{eq:margin} give
\[
 \ip e{\Psi(e)}
 =\E[(a^\top e)(S_{y^*}(a)-S_y(a))]
 \ge\lambda_1\E[(a^\top e)^2]
 \ge\lambda_0\lambda_1\norm e^2.
\]
Also,
\(\norm{\Psi(e)}\le\lambda_2\E[\norm a\,|a^\top e|]\le
\lambda_2\E\norm a^2\,\norm e\).
 Since \(\ip e{F(y^*)}\ge0\), integration along the segment from \(y^*\) to \(y\) gives
\[
 f_d(y)-f_d(y^*)\ge F(y^*)^\top e+\frac{\mulo}{2}\norm e^2.
\]
 The finite convex objective is absolutely continuous along the segment, with derivative equal almost everywhere to the selected subgradient's directional component.
The bound gives uniqueness on \(\Y\); all minimizers lie there, so uniqueness is global.

For the reward comparison, the hinge identity
\((r-a^\top y)_+=(r-a^\top y)X(y;r,a)\) gives
\[
 f_d(y)=R_*(y)+y^\top F(y).
\]
Since \(y^{*\top}F(y^*)=0\), this first gives
\(R_*(y^*)=f_d(y^*)\). Moreover,
\[
 R_*(y^*)-R_*(y)-y^{*\top}F(y)
 =f_d(y^*)-f_d(y)+e^\top F(y).
\]
Because \(F(y)\in\partial f_d(y)\) and
\(F(y^*)\in\partial f_d(y^*)\), the two subgradient inequalities imply
\[
 0\le f_d(y^*)-f_d(y)+e^\top F(y)
 \le e^\top\big(F(y)-F(y^*)\big)
 =e^\top\Psi(e)
 \le L\norm e^2.
\]
This proves \eqref{eq:quadratic}.
\end{proof}

Weak duality holds for every realized sample and fixed \(y\ge0\):
\begin{equation}
 \OPT_T^{\rm H}\le Td^\top y+\sum_{t=1}^T(r_t-a_t^\top y)_+,
 \qquad \E\OPT_T^{\rm H}\le Tf_d(y^*)=TR_*(y^*). \label{eq:weak-dual}
\end{equation}
It is therefore enough to control the policy's loss relative to this population upper bound.

The cap bounds the learning input even when nonbinding inventory accumulates:
\begin{equation}
 \norm{g_t}\le G:=\sqrt m\,(\dmax+1+\bar a)
 \quad\text{on every sample path}. \label{eq:gradient-bound}
\end{equation}

For the stopping construction and the subsequent proofs, retain \(M,\epsT,e_t,u_t\) from Section~\ref{sec:roadmap}, and write
\[
 z_t=y_{t,\N},\qquad \zeta_t=h(y_t)-a_t\widetilde x_t.
\]
Then \(\E_t\zeta_t=0\) and \(\norm{\zeta_t}\le2\sqrt m\,\bar a\) on every path, including when the feasibility filter is active.

\textbf{The safe region and analytical stopping time.}\label{sec:stop}
Choose fixed proof-only radii
\begin{equation}
 0<\delta_B<\min\{1/4,\dmin/4\},\qquad
 0<\delta_i<\min\{1/4,\dmin/4,s_i/4\}\quad(i\in\N), \label{eq:deltas}
\end{equation}
and choose \(\rho>0\) such that
\begin{equation}
 \rho<\min\left\{1,\frac{\sigma_0}{4},
       \frac{\min_{i\in\B}y_i^*}{4},
       \frac{\min_{i\in\N}s_i}{8(1+L)}\right\}. \label{eq:rho}
\end{equation}
 Also require \(\delta_B<\rho/(4\kappa)\), shrinking \(\delta_B\) after choosing \(\rho\) if necessary.

 Fix an integer \(H\), independent of \(T\), such that
\begin{equation}
 H\ge s_0+2,\qquad
 \kappa G/H\le \rho/2,\qquad
 H(3\dmin/4)>\bar a,\qquad H\ge8. \label{eq:H}
\end{equation}
 For now assume \(T>4H\). The radii and \(H\) are proof-only quantities; bounded horizons are handled in the final step.

The early stopping index \(S\) is the first period-start \(t\le M\) for which
\[
 \norm{u_t}\ge\delta_B
 \quad\text{or}\quad q_{t,i}\le d_i-\delta_i\text{ for some }i\in\N.
\]
Set \(S=M+1\) if there is no such index. At the midpoint, define the entrance event
\begin{equation}
 \G=\{S>M,\ \norm{e_M}<\rho/2,\ \norm{u_M}<\delta_B/2,\
                  q_{M,i}\ge d_i+s_i/4\ (i\in\N)\}. \label{eq:G}
\end{equation}
 The nonbinding surplus in \(\G\) buffers later lower crossings.

On \(\G\), let \(\tau\) be the first period-start \(t\ge M\) at which
\begin{equation}
 \norm{u_t}\ge\delta_B,\quad
 \norm{e_t}\ge \rho,\quad
 q_{t,i}\le d_i-\delta_i\text{ for some }i\in\N,\quad
 \text{or }n_t\le H. \label{eq:late-stop}
\end{equation}

Set \(\tau=S\) on \(\{S\le M\}\), and \(\tau=M\) on \(\{S>M\}\setminus\G\).
Since \(\G\in\F_{M-1}\), this is a stopping time for \(\mathcal H_t=\F_{t-1}\);
\(\1\{t<\tau\}\) is predictable and preserves conditional centering of \(\zeta_t\).
The implemented policy continues after \(\tau\).

\textbf{Updates before the stop.}
For \(t<\tau\), \(B_{t,i}\ge n_t(3\dmin/4)>\bar a\), so \(x_t=\widetilde x_t\), including the exit-producing update.
For \(M\le t<\tau\), \(\alpha_t=\kappa/(n_t-1)\) and \(\alpha_tG\le\rho/2\).
The unprojected candidate is within \(3\rho/2\) of \(y^*\); positive parts cannot increase this distance.
By \eqref{eq:rho}, radial projection is inactive and binding coordinates remain positive.
Only nonbinding orthant reflection can act, also when \(t+1=\tau\).

\begin{proof}[Proof of Lemma~\ref{lem:regret-reduction}]
Write \(R^*=R_*(y^*)\). By \eqref{eq:weak-dual}, it suffices to bound
\(TR^*-\E\sum_{t=1}^T r_tx_t\).
Since \(|R^*|,|r_tx_t|\le\bar r\), the periods from \(\tau\) onward contribute at most \(2\bar r\,\E n_\tau\).

Before \(\tau\), the filter is inactive and the stopping indicator is predictable. The reward comparison \eqref{eq:quadratic} and the tower property give
\begin{equation}
 \E\sum_{t<\tau}(R^*-r_t\widetilde x_t)
 \le y^{*\top}\E\sum_{t<\tau}(d-a_t\widetilde x_t)
     +L\E\sum_{t<\tau}\norm{e_t}^2. \label{eq:prestop-regret}
\end{equation}
 Before \(\tau\), preferred and implemented actions coincide, so
\begin{equation}
 \sum_{t<\tau}(d-a_t\widetilde x_t)
 =B_\tau-n_\tau d=n_\tau(q_\tau-d). \label{eq:telescope}
\end{equation}
 Only binding coordinates contribute to \(y^{*\top}(q_\tau-d)\). Their exit overshoot is controlled by
\[
 q_{t+1,\B}-q_{t,\B}
 =\frac{q_{t,\B}-a_{t,\B}\widetilde x_t}{n_t-1},
\]
whose numerator is bounded on safe states.
A late exit has \(n_t-1\ge H\), giving \(\norm{q_{\tau,\B}-d_\B}\le\delta_B+C/H\).
An early exit has overshoot at most \(C/T\), and a midpoint failure without an early exit remains inside the early resource region.
Thus the linear term is at most \(C\E n_\tau\) in every case. Combining the bounds proves the lemma.
\end{proof}

\subsection{Proofs for Step 1: Early Learning and Midpoint Entrance}\label{app:early}\label{app:forward}\label{app:early-localization}
We first establish the moment and localization bounds used in Lemma~\ref{lem:early}, then combine them to obtain its three conclusions.
For \(1\le t\le M\), define
\[
 E_t=\E[\1\{t<S\}\norm{e_t}^2],\qquad
 U_t=\E[\1\{t<S\}\norm{u_t}^2].
\]
These moments discard paths after the early stop.
The localization estimates below instead retain the state at the stop.

\begin{lemma}\label{lem:early-moments}
For \(1\le t\le M\),
\begin{equation}
 E_t\le C\big((t+s_0)^{-1}+\epsT\big),\qquad
 U_t\le C\epsT,\qquad
 \sum_{t=1}^{M-1}E_t\le C\log T. \label{eq:early-bounds}
\end{equation}
\end{lemma}
\begin{proof}
For \(T>4H\), we have \(s_0<M\).
Define
\begin{equation}
 \begin{aligned}
 D&=R+D_0,& K_0&=2s_0(D^2+\kappa^2G^2),\\
 K_T&=K_0+\kappa^2G^2\log T,&
 A_0&=(s_0+1)D^2+4\kappa^2G^2.
 \end{aligned}
 \label{eq:named-early}
\end{equation}

Before \(S\), the binding cap is inactive, so \(q^c_{t,\B}=q_{t,\B}\). On a nonbinding coordinate,
\(q^c_{t,i}-d_i+s_i\ge s_i-\delta_i>0\).
Since \(z_t\ge0\), Lemma~\ref{lem:geometry} yields
\begin{align}
 \ip{e_t}{q^c_t-h(y_t)}
 &\ge\mulo\norm{e_t}^2+\ip{e_{t,\B}}{u_t}\notag\\
 &\ge\frac{\mulo}{2}\norm{e_t}^2
       -\frac{1}{2\mulo}\norm{u_t}^2. \label{eq:early-coercivity}
\end{align}
 Projection nonexpansiveness and \eqref{eq:gradient-bound} give
\[
 \E_t\norm{e_{t+1}}^2
 \le(1-\mulo\alpha_t)\norm{e_t}^2
     +\alpha_t\norm{u_t}^2/\mulo+\alpha_t^2G^2.
\]
For \(t<M\),
\(\kappa/(t+s_0)\le\alpha_t\le2\kappa/(t+s_0)\) and
\(\mulo\alpha_t\le1/4\).
 Multiply by \(\1\{t<S\}\), condition, and use \(\1\{t+1<S\}\le\1\{t<S\}\) on the nonnegative next-state square to obtain
\begin{equation}
 E_{t+1}\le\left(1-\frac{2}{t+s_0}\right)E_t
       +\frac{b}{t+s_0}U_t+\frac{c}{(t+s_0)^2}. \label{eq:forward-E}
\end{equation}
Here one may take \(b=2\kappa/\mulo=\kappa^2\) and
\(c=4\kappa^2G^2\).

 For \(t<S\), \(t<M\), the filter is inactive even on an exit-producing step, so
\begin{equation}
 u_{t+1}=u_t+\beta_t[u_t+\Psi_\B(e_t)+\zeta_{t,\B}],
 \qquad \beta_t=(T-t)^{-1}\le2/T. \label{eq:early-u}
\end{equation}

 The shared sample in the two updates motivates the joint energy, for \(s_0\le t\le M\),
\begin{equation}
 \mathcal V_t=\norm{e_t}^2+
 \frac{t-1}{T}\left\|e_{t,\B}+\frac{\kappa n_t}{t-1}u_t\right\|^2,
 \qquad v_t=\E[\1\{t<S\}\mathcal V_t].
 \label{eq:early-energy}
\end{equation}
We will establish the drift bound
\begin{equation}
 t v_{t+1}+E_t\le (t-1)v_t+\frac{\kappa^2G^2}{t},
 \qquad s_0\le t<M.
 \label{eq:early-energy-drift}
\end{equation}

\textbf{The early projection.}
For \(0\le\theta\le1\), set
\(\norm{x}_\theta^2=(1+\theta)\norm{x_\B}^2+\norm{x_\N}^2\),
with its induced inner product \(\ip{\cdot}{\cdot}_\theta\).
 For a center supported on \(\B\), the Euclidean projection satisfies
\begin{equation}
 \norm{\Pi_\Y(x)-v}_\theta^2\le\norm{x-v}_\theta^2
 \quad\text{if }v\in\Y\text{ and }v_\N=0.
 \label{eq:weighted-projection}
\end{equation}
Taking positive parts decreases every coordinate's distance to \(v\).
For the radial step, let \(a=\norm z\ge R\) and
\(b=\norm{z_\B}\le a\). Since \(\norm v\le R\) and \(v_\N=0\),
\[
 \ip{z}{z-v}_\theta
 \ge a^2+\theta b^2-(1+\theta)Rb
 \ge(a-b)(a-\theta b)\ge0.
\]
 The weighted distance is nondecreasing with radius outside \(R\), which proves \eqref{eq:weighted-projection}. The claim is restricted to centers with \(v_\N=0\).

\textbf{Cancellation and dissipation.}
On \(t<S\), \(s_0\le t<M\), put
\(\bar e_{t+1}=e_t-\kappa g_t/t\).
The binding inventory update and the early stepsize give the exact identity
\begin{equation}
 \bar e_{t+1,\B}+\frac{\kappa(T-t)}{t}u_{t+1}
 =e_{t,\B}+\frac{\kappa(T-t)}{t}u_t.
 \label{eq:early-cancellation}
\end{equation}
 Completing the square in \(\mathcal V_{t+1}\) gives a weighted distance with \(\theta=t/T\) and center
\[
 v=y^*-\frac{\kappa(T-t)}{T+t}(u_{t+1},0_\N).
\]
 The vector \((u,0_\N)\) has zero nonbinding coordinates. On every safe update, including the exit-producing one,
\[
 \norm{u_{t+1}}\le\delta_B+\frac{2G}{T}
 <\frac{\rho}{2\kappa},
\]
by \(\delta_B<\rho/(4\kappa)\), \(T>4H\), and \eqref{eq:H}.
Thus \(\norm{v-y^*}<\rho/2\), so \(v\in\Y\) and \(v_\N=0\).
Applying \eqref{eq:weighted-projection} and \eqref{eq:early-cancellation}
therefore yields
\begin{equation}
 \mathcal V_{t+1}\le
 \norm{e_t-\kappa g_t/t}^2+
 \frac{t}{T}\left\|e_{t,\B}+\frac{\kappa(T-t)}{t}u_t\right\|^2.
 \label{eq:early-unprojected-energy}
\end{equation}
Let \(\chi_t=(2t-1)/T\in[0,1]\).
Expand the right side, use \(\kappa\mulo=2\) and the first
inequality in \eqref{eq:early-coercivity}, and condition on the current history:
\begin{align*}
 t\E_t\mathcal V_{t+1}-(t-1)\mathcal V_t
 &\le -(3-\chi_t)\norm{e_{t,\B}}^2-3\norm{z_t}^2
       -2\kappa\chi_t\ip{e_{t,\B}}{u_t}\\
 &\qquad-\kappa^2(2-\chi_t)\norm{u_t}^2
       +\frac{\kappa^2G^2}{t}\\
 &\le -\norm{e_t}^2+\frac{\kappa^2G^2}{t}.
\end{align*}

The last step uses
\(-2\kappa\chi_t\ip{e_{t,\B}}{u_t}\le\chi_t\norm{e_{t,\B}}^2+\kappa^2\chi_t\norm{u_t}^2\)
and \(\chi_t\le1\).
Multiplication by the predictable indicator \(\1\{t<S\}\), followed by expectation and deletion of nonnegative exit-state energy, gives \eqref{eq:early-energy-drift}.

\textbf{Initialization and the three estimates.}
 Projection gives the global bound \(\norm{e_t}\le D\).
For \(t\le s_0\), actual inventory accounting and bounded implemented
consumption give \(\norm{u_t}\le(t-1)G/n_t\), hence
\((s_0-1)v_{s_0}\le2s_0(D^2+\kappa^2G^2)=K_0\).
Summing \eqref{eq:early-energy-drift} gives, for \(s_0\le t\le M\),
\begin{equation}
 (t-1)v_t+\sum_{j=s_0}^{t-1}E_j
 \le K_0+\kappa^2G^2\sum_{j=s_0}^{t-1}j^{-1}\le K_T.
 \label{eq:early-energy-sum}
\end{equation}
The omitted initial price terms are at most \(s_0D^2\).
A second completion of the square gives
\[
 \mathcal V_t\ge
 \frac{\kappa^2 n_t^2}{(t-1)(T+t-1)}\norm{u_t}^2
 \ge\frac{\kappa^2T}{6(t-1)}\norm{u_t}^2,
\]
so \(U_t\le6K_T/(\kappa^2T)\) for \(s_0\le t\le M\).
For \(t<s_0\), the same bound follows from
\(U_t\le4s_0^2G^2/T^2\), \(K_T\ge2s_0\kappa^2G^2\), and
\(T>4s_0\).
Finally, in \eqref{eq:forward-E}, substitute this uniform bound on \(U_t\).

The supersolution \(A_0/(t+s_0)+3K_T/T\) dominates \(E_1\).
Its constant part cancels the inventory term because \(b=\kappa^2\), and its reciprocal part dominates the noise because \(A_0\ge4\kappa^2G^2\).
Induction yields \eqref{eq:explicit-pointwise}; together with \eqref{eq:early-energy-sum}, this proves \eqref{eq:early-bounds}.
\end{proof}

For the localization argument and later maximum estimates, we use the following martingale bound.
For a square-integrable vector martingale, Doob\textquotesingle s inequality is
\begin{equation}
 \E\max_{a\le k\le b}\norm{L_k}^2\le4\E\norm{L_b}^2, \qquad L_a=0. \label{eq:Doob}
\end{equation}
One direct proof applies the submartingale property of \(\norm{L_k}\) at the first crossing of level \(\lambda\):
\[
 \lambda\,\Prob(\max_k\norm{L_k}\ge\lambda)
 \le\E[\norm{L_b}\1\{\max_k\norm{L_k}\ge\lambda\}].
\]
Integrating twice the two sides with respect to \(\lambda\) gives
\(\E\max_k\norm{L_k}^2\le2\E[\norm{L_b}\max_k\norm{L_k}]\).
Cauchy--Schwarz proves \eqref{eq:Doob}.

Conditional orthogonality sums the noise variances. Predictable stopping indicators and entrance events measurable before the first increment preserve both this identity and the maximal bound.

\begin{lemma}\label{lem:early-max}
The early stop and midpoint state satisfy
\begin{align}
 \E\max_{1\le k\le M}\norm{q_{k\wedge S,\B}-d_\B}^2&\le C\epsT,\label{eq:early-max}\\
 \Prob(S\le M)&\le C\epsT,\label{eq:early-exit}\\
 \Prob(S>M,\ q_{M,i}<d_i+s_i/4)&\le C\epsT\quad(i\in\N).\label{eq:mid-N}
\end{align}
\end{lemma}
\begin{proof}[Proof of Lemma~\ref{lem:early-max}]
Through the inclusive early exit,
\begin{equation}
 q_{k\wedge S}-d=
 \frac{\sum_{t<k\wedge S}\big(F(y^*)+\Psi(e_t)+\zeta_t\big)}
 {n_{k\wedge S}},\qquad 1\le k\le M, \label{eq:early-inventory-identity}
\end{equation}
and the denominator is at least \(T/2\).
Cauchy--Schwarz and Lemma~\ref{lem:early-moments} imply
\[
 \E\max_{k\le M}\norm{\sum_{t<k\wedge S}\Psi(e_t)}^2
 \le M L^2\sum_{t<M}E_t\le CT\log T.
\]
 The stopped noise sum is a bounded-increment vector martingale. The \(L^2\) maximal inequality gives
\[
 \E\max_{k\le M}\norm{\sum_{t<k\wedge S}\zeta_t}^2\le CT.
\]
On \(\B\), the deterministic term \(F_\B(y^*)\) vanishes, proving \eqref{eq:early-max}.
On \(\N\), its sum is nonnegative. A lower crossing by \(\delta_i\) therefore requires the centered numerator in \eqref{eq:early-inventory-identity} to have magnitude at least \(\delta_iT/2\).
Markov's inequality and a finite union bound prove \eqref{eq:early-exit}.

On \(S>M\), the deterministic nonbinding increment at \(M\) equals
\((M-1)s_i/n_M\ge s_i/2\).
If \(q_{M,i}<d_i+s_i/4\), the centered numerator must again have negative magnitude of order \(T\).
The same maximal bound proves \eqref{eq:mid-N}.
\end{proof}

\begin{proof}[Proof of Lemma~\ref{lem:early}]
The pointwise estimate \eqref{eq:early-pointwise} is the first bound in \eqref{eq:early-bounds}.
For \(t<M\), the events \(\{t<\tau\}\) and \(\{t<S\}\) coincide.
Thus \eqref{eq:early-cumulative} follows from the cumulative bound in \eqref{eq:early-bounds}.

By the definition of \(\G\), unsuccessful entrance can occur through an early resource exit, excessive midpoint price or binding resource error, or insufficient nonbinding surplus.
The first and last events are bounded by \eqref{eq:early-exit} and \eqref{eq:mid-N}.
On \(\{S>M\}\), Markov's inequality bounds the two midpoint-error probabilities by \(4E_M/\rho^2\) and \(4U_M/\delta_B^2\).
Both are \(O(\epsT)\) by \eqref{eq:early-bounds}, since \(M\) is of order \(T\).
A finite union bound proves \eqref{eq:G-prob}.
Finally, \(\G\subseteq\{S>M\}\) gives
\[
 \E[\1_\G(\norm{e_M}^2+\norm{u_M}^2)]
 \le E_M+U_M\le C\epsT,
\]
which is \eqref{eq:G-moment}.
\end{proof}

\subsection{Proofs for Step 2: Late Stabilization through Matched Updates}\label{app:late}\label{app:restoring}\label{app:reflection}\label{app:reverse}
We first prove the restoring and reflection estimates used to control binding inventory and nonbinding prices.
We then establish the invariant and combine the estimates in the proof of Lemma~\ref{lem:late}.

\begin{lemma}\label{lem:restoring}
For every \(e=(e_\B,z)\) with \(y^*+e\in\Y\) and every \(u,w\) satisfying \(e_\B+\kappa u=w\), there are finite constants \(C_w,C_z\), depending only on \(\mulo,L,\kappa\), such that
\begin{equation}
 \ip u{u+\Psi_\B(e)}
 \le-\frac18\norm u^2+C_w\norm w^2+C_z\norm z^2. \label{eq:restoring}
\end{equation}
\end{lemma}
\begin{proof}[Proof of Lemma~\ref{lem:restoring}]

Using \(e_\B=w-\kappa u\),
\begin{align*}
 \ip u{u+\Psi_\B(e)}
 &=\norm u^2-\frac1\kappa\ip e{\Psi(e)}
      +\frac1\kappa\ip{(w,z)}{\Psi(e)}\\
 &\le\norm u^2-\frac{3\mulo}{4\kappa}\norm e^2
      +\frac{L^2}{\mulo\kappa}(\norm w^2+\norm z^2).
\end{align*}
The inequality uses \eqref{eq:geometry} together with
\[
 L\norm{(w,z)}\norm e
 \le \frac{\mulo}{4}\norm e^2
      +\frac{L^2}{\mulo}(\norm w^2+\norm z^2).
\]
Moreover,
\[
 \norm e^2=\norm{w-\kappa u}^2+\norm z^2
 \ge \frac34\kappa^2\norm u^2-3\norm w^2+\norm z^2.
\]
Consequently,
\[
 \ip u{u+\Psi_\B(e)}
 \le -\frac18\norm u^2
 +\left(\frac{9\mulo}{4\kappa}+\frac{L^2}{\mulo\kappa}\right)\norm w^2
 +\frac{L^2}{\mulo\kappa}\norm z^2.
\]
Indeed, before substituting the gain, the coefficient on \(\norm u^2\) is
\(1-9\mulo\kappa/16\), which equals \(-1/8\) when \(\mulo\kappa=2\).
Thus Lemma~\ref{lem:restoring} holds, for example, with
\[
 C_w=\frac{9\mulo}{4\kappa}+\frac{L^2}{\mulo\kappa},\qquad
 C_z=\frac{L^2}{\mulo\kappa}.
\]
\end{proof}

\begin{lemma}\label{lem:reflection}
Let \(Z_{t+1}=[Z_t-\eta_t b_t+\eta_t\xi_t]_+\), where \(Z_a\ge0\) is square integrable, \(b_t\ge0\) is predictable, \(\eta_t\ge0\) is deterministic, and \(\xi_t\) is a martingale difference satisfying \(|\xi_t|\le V_\xi\) almost surely. The recursion may be frozen at a stopping time. Then, including the state produced by the last update,
\begin{equation}
 \E\max_{a\le k\le b}Z_k^2
 \le2\E Z_a^2+32V_\xi^2\sum_{t=a}^{b-1}\eta_t^2. \label{eq:reflection-max}
\end{equation}
For an entrance event \(A\in\F_{a-1}\),
\[
 \E\!\left[\1_A\max_{a\le k\le b}Z_k^2\right]
 \le2\E[\1_AZ_a^2]+32V_\xi^2\Prob(A)\sum_{t=a}^{b-1}\eta_t^2.
\]
\end{lemma}
\begin{proof}[Proof of Lemma~\ref{lem:reflection}]
Removing the nonpositive drift gives an upper comparison
\(R_{t+1}=[R_t+\eta_t\xi_t]_+\), \(R_a=Z_a\);
the update is monotone in the state.
With \(L_k=\sum_{t=a}^{k-1}\eta_t\xi_t\), the elementary reflection identity gives
\[
 R_k=Z_a+L_k-\min\left\{0,\min_{a\le s\le k}(Z_a+L_s)\right\}
 \le Z_a+2\max_{a\le s\le k}|L_s|.
\]
Therefore \(\max R_k^2\le2Z_a^2+8\max|L_k|^2\).
The \(L^2\) martingale maximal inequality and orthogonality give
\(\E\max|L_k|^2\le4\E L_b^2\le4V_\xi^2\sum\eta_t^2\).
For freezing, insert the predictable survival indicator into both the drift and the martingale increment.
For an entrance event, multiply the stopped martingale by that event, which is measurable before its first increment.
\end{proof}

\begin{proof}[Proof of Lemma~\ref{lem:late}]
On \(\G\), set \(w=e_{M,\B}+\kappa u_M\).
The entrance bound \eqref{eq:G-moment} implies
\begin{equation}
 W:=\E[\1_\G\norm w^2]\le C\epsT. \label{eq:w}
\end{equation}
For \(M\le t<\tau\), the cap and the relevant projections are inactive on binding coordinates, and \(\alpha_t=\kappa\beta_t\), where \(\beta_t=(n_t-1)^{-1}\). The updates are therefore
\begin{align}
 e_{t+1,\B}&=e_{t,\B}-\kappa\beta_t
       [u_t+\Psi_\B(e_t)+\zeta_{t,\B}],\label{eq:binding-e}\\
 u_{t+1}&=u_t+\beta_t[u_t+\Psi_\B(e_t)+\zeta_{t,\B}],
       \qquad\beta_t=(n_t-1)^{-1}.\label{eq:binding-u}
\end{align}
Adding the two recursions cancels their common response and noise, proving \eqref{eq:invariant}.
The update properties in Appendix~\ref{app:setup} hold through the exit-producing step, so the identity remains valid at \(t=\tau\).

For \(i\in\N\), a safe late state satisfies
\begin{equation}
 b_{t,i}:=q^c_{t,i}-h_i(y_t)
 \ge s_i-\delta_i-L\rho\ge s_i/2>0. \label{eq:N-drift}
\end{equation}
Consequently,
\begin{equation}
 z_{t+1,i}=[z_{t,i}-\alpha_t(b_{t,i}+\zeta_{t,i})]_+
 \quad(M\le t<\tau). \label{eq:z-reflected}
\end{equation}

Apply Lemma~\ref{lem:reflection} to \(z_{k\wedge\tau}\) on \(\G\), with \(\xi_t=-\zeta_{t,i}\). For \(t(n)=T-n+1\) and \(H\le n\le n_M\),
\[
 \sum_{t=M}^{t(n)-1}\alpha_t^2
 =\kappa^2\sum_{k=n}^{n_M-1}k^{-2}\le C/n.
\]
Together with \eqref{eq:G-moment}, this yields
\begin{equation}
 \E\left[\1_\G\max_{M\le k\le t(n)}\norm{z_{k\wedge\tau}}^2\right]
 \le C(\epsT+1/n). \label{eq:z-global-max}
\end{equation}
This maximum includes the frozen exit state and controls the \(z\)-term in \eqref{eq:restoring}.

Define the killed binding moment
\[
 A_n=\E[\1_\G\1\{t(n)<\tau\}\norm{u_{t(n)}}^2].
\]
The binding increment field is bounded before \(\tau\). Squaring \eqref{eq:binding-u}, conditioning, and using Lemma~\ref{lem:restoring} gives
\begin{align}
 A_{n-1}
 &\le\left(1-\frac{1}{4(n-1)}\right)A_n
 +\frac C{n-1}
   \left(W+\E[\1_\G\1\{t(n)<\tau\}\norm{z_{t(n)}}^2]\right)
 +\frac C{(n-1)^2}\notag\\
 &\le\left(1-\frac{1}{4(n-1)}\right)A_n
       +\frac{C\epsT}{n-1}+\frac C{(n-1)^2}. \label{eq:reverse-A}
\end{align}
 As in the early recursion, bound the next killed moment using the current predictable survival indicator before conditioning. The coefficient is nonnegative for \(n>H\ge8\).

The following comparison solves this recursion.
For completeness, suppose
\[
 A_{n-1}\le\left(1-\frac a{n-1}\right)A_n
            +\frac{b\varepsilon}{n-1}+\frac c{(n-1)^2},
 \qquad H<n\le N,
\]
where \(a>0\), \(H>a+1\), and
\(A_N\le C_0(\varepsilon+1/N)\).
The comparison \(K(\varepsilon+1/n)\) is a supersolution if
\(K\ge C_0\), \(aK\ge b\), and \((a+1)K\ge2c\).
Indeed,
\begin{align*}
 &K\left(\varepsilon+\frac1{n-1}\right)
 -\left(1-\frac a{n-1}\right)K\left(\varepsilon+\frac1n\right)
 -\frac{b\varepsilon}{n-1}-\frac c{(n-1)^2}\\
 &=\frac{(aK-b)\varepsilon}{n-1}
   +\frac{(a+1)K}{n(n-1)}-\frac c{(n-1)^2}\ge0,
\end{align*}
since \(n/(n-1)\le2\).
Backward induction proves \(A_n\le K(\varepsilon+1/n)\).

Apply this comparison with \(a=1/4\), \(\varepsilon=\epsT\), and \(N=n_M\).
The entrance bound \eqref{eq:G-moment} supplies the required initial condition.
It follows that
\begin{equation}
 A_n\le C(\epsT+1/n). \label{eq:A-rate}
\end{equation}
The invariant \eqref{eq:invariant}, its entrance moment \eqref{eq:w}, and the nonbinding estimate \eqref{eq:z-global-max} then imply
\begin{equation}
 \E[\1_\G\1\{t(n)<\tau\}\norm{e_{t(n)}}^2]
 \le C(\epsT+1/n). \label{eq:e-late}
\end{equation}
Adding \eqref{eq:A-rate} and \eqref{eq:e-late} gives \eqref{eq:late-bounds}.
On \(\G^c\), \(\tau\le M\), so there is no late contribution before the stop.
On \(\G\), summing \eqref{eq:e-late} over \(M\le t<\tau\) contributes at most
\(C(T\epsT+\sum_{n=H}^{n_M}n^{-1})\le C\log T\).
Together with \eqref{eq:early-cumulative}, this proves \eqref{eq:cumulative}.
\end{proof}

\subsection{Proofs for Step 3: Controlling Premature Exits}\label{app:exit}\label{app:block}
We retain the exit state in all maxima used to control first exits.
Write \(t(n)=T-n+1\) and \(\beta_t=(n_t-1)^{-1}\).
Set \(N_0=n_M\), \(N_{j+1}=\lfloor N_j/2\rfloor\), and \(J=\min\{j:N_j\le2H\}\).
For \(j<J\), let \(a_j=t(N_j)\), \(b_j=t(N_{j+1})\).
The block includes states \([a_j,b_j]\) and updates \([a_j,b_j)\); an exit at \(b_j\) belongs to this block.
For \(N=N_j\),
\begin{equation}
 \sum_{t=a_j}^{b_j-1}\beta_t\le C,\qquad
 \sum_{t=a_j}^{b_j-1}\beta_t^2\le C/N. \label{eq:block-sums}
\end{equation}

\begin{lemma}\label{lem:block}
Let \(A=\G\cap\{a_j<\tau\}\). Then
\begin{equation}
 \E\left[\1_A\max_{a_j\le k\le b_j}
 \left(\norm{u_{k\wedge\tau}}^2+\norm{e_{k\wedge\tau}}^2\right)\right]
 \le C(\epsT+1/N_j). \label{eq:block-max}
\end{equation}
\end{lemma}
\begin{proof}[Proof of Lemma~\ref{lem:block}]
Write \(D_t=u_t+\Psi_\B(e_t)\).
On the entrance event \(A\),
\begin{equation}
 u_{k\wedge\tau}=u_{a_j}
 +\sum_{t=a_j}^{k-1}\1\{t<\tau\}\beta_tD_t
 +\sum_{t=a_j}^{k-1}\1\{t<\tau\}\beta_t\zeta_{t,\B}.
 \label{eq:vector-block}
\end{equation}
 The sum includes the exit-producing update. The entrance square is bounded by \(C(\epsT+1/N)\) using \eqref{eq:A-rate}; weighted Cauchy--Schwarz bounds the drift by
\begin{align*}
 &\E\left[\1_A
    \left(\sum_{t=a_j}^{b_j-1}\1\{t<\tau\}\beta_t\norm{D_t}\right)^2\right]\\
 &\hspace{1em}\le
 \left(\sum_{t=a_j}^{b_j-1}\beta_t\right)
 \sum_{t=a_j}^{b_j-1}\beta_t
       \E[\1_A\1\{t<\tau\}\norm{D_t}^2]
 \le C(\epsT+1/N).
\end{align*}
The last inequality uses
\(\norm{D_t}^2\le C(\norm{u_t}^2+\norm{e_t}^2)\),
\eqref{eq:A-rate}--\eqref{eq:e-late}, and \(n_t\asymp N\) on the block.
The noise term is a vector martingale after multiplication by \(A\), since \(A\in\F_{a_j-1}\).
Its expected squared maximum is at most \(C\sum\beta_t^2\le C/N\).
Applying \(\norm{x+y+z}^2\le3(\norm x^2+\norm y^2+\norm z^2)\) to
\eqref{eq:vector-block} proves the \(u\) maximum.

 The binding invariant, \eqref{eq:w}, and the nonbinding maximum \eqref{eq:z-global-max} through \(b_j\) then give the price bound in \eqref{eq:block-max}.
\end{proof}

\begin{proof}[Proof of Lemma~\ref{lem:remaining}]
If a binding-rate or full-price exit first occurs in \((a_j,b_j]\), the maximum in \eqref{eq:block-max} exceeds the square of a fixed radius. Hence
\begin{equation}
 \Prob\big(\G,\ a_j<\tau\le b_j,\ \tau\text{ is a binding or price exit}\big)
 \le C(\epsT+1/N_j). \label{eq:binding-exit-prob}
\end{equation}

\textbf{Lower exits of nonbinding inventory.}
 For \(i\in\N\), the uncapped physical resource rate has positive drift before \(\tau\):
\begin{align}
 q_{t+1,i}-q_{t,i}
 &=\beta_t[q_{t,i}-h_i(y_t)]+\beta_t\zeta_{t,i},\label{eq:N-resource}\\
 q_{t,i}-h_i(y_t)
 &\ge q^c_{t,i}-h_i(y_t)\ge s_i/2>0.\label{eq:N-resource-drift}
\end{align}
On \(\G\), the starting rate is at least \(d_i+s_i/4\). Therefore, if this coordinate causes a first lower exit, the stopped martingale
\[
 Z_k^i=\sum_{t=M}^{k-1}\1_\G\1\{t<\tau\}\beta_t\zeta_{t,i}
\]
must be at most \(-(\delta_i+s_i/4)\) at that exit. Its variance through \(b_j\) satisfies
\[
 \E(Z_{b_j}^i)^2
 \le C\sum_{t=M}^{b_j-1}(n_t-1)^{-2}\le C/N_j.
\]
 Accumulating noise from the midpoint covers crossings caused across multiple blocks. Doob\textquotesingle s inequality and a union bound give
\begin{equation}
 \Prob\big(\G,\ a_j<\tau\le b_j,\ \tau\text{ is a nonbinding lower exit}\big)
 \le C/N_j. \label{eq:N-exit-prob}
\end{equation}
Simultaneous exit causes are covered by the same union bound.
On \(\G\cap\{a_j<\tau\le b_j,\ n_\tau>2H\}\), the terminal-window condition cannot cause the stop.
Combining \eqref{eq:binding-exit-prob} and \eqref{eq:N-exit-prob} therefore proves \eqref{eq:block-exit}.

For the remaining-horizon bound \eqref{eq:remaining}, sum the exit charges as follows.
Paths outside \(\G\), including early exits, contribute at most
\(T\Prob(\G^c)\le C\log T\).
On \(\G\), the terminal portion with \(n_\tau\le2H\) contributes at most \(2H\). For all other paths, \eqref{eq:binding-exit-prob}--\eqref{eq:N-exit-prob} imply
\[
 \E[\1_\G n_\tau]
 \le2H+C\sum_{j<J}N_j(\epsT+1/N_j)
 \le2H+C(T\epsT+J)\le C\log T,
\]
because \(\sum_jN_j\le2n_M\le T+1\) and \(J=O(\log T)\).
\end{proof}

\subsection{Proof of the Lower-Bound Corollary}\label{app:scope}\label{app:lower-bound}
\begin{proof}[Proof of Corollary~\ref{cor:lower-bound}]
Take $\bar r=\bar a=1$, $\underline d=1/2$, and $\bar d=1$, which verify (\ref{ass:input}).
The second-moment matrix is
\[
 \E[aa^\top]=\begin{pmatrix}1&1/2\\1/2&1/2\end{pmatrix}
 \succeq\frac{3-\sqrt5}{4}I_2,
\]
so (\ref{ass:moment}) holds. At $y^*=(1/2,0)^\top$,
\[
 h(y^*)=(1/2,1/4)^\top,\qquad d-h(y^*)=(0,3/4)^\top.
\]
The convex optimality conditions make \(y^*\) a minimizer of \(f_d\). The displayed response satisfies (\ref{ass:complementarity}), with \(\B=\{1\}\) and \(\N=\{2\}\).

To check (\ref{ass:threshold}) on its entire domain, note that $\Xi_1=\{y\ge0:\norm y\le3\}$.
For either possible $a$, put $u=a^\top y$. Then $a^\top y^*=1/2$ and $0\le u\le3\sqrt2<9/2$.
Independence and the uniform reward law give
\[
 \left|\Prob(r\ge u\mid a)-\Prob(r\ge1/2\mid a)\right|
 =\begin{cases}
 |u-1/2|,&0\le u\le1,\\
 1/2,&1<u\le3\sqrt2.
 \end{cases}
\]
This lies between $|u-1/2|/8$ and $|u-1/2|$. Hence (\ref{ass:threshold}) holds with $\lambda_1=1/8$ and $\lambda_2=1$.

For even $T$, the first resource permits at most $T/2$ acceptances. The second resource has capacity $T$ and is redundant, since $Z_t\le1$.
The fractional hindsight optimum thus equals the sum of the largest $T/2$ rewards, exactly the integer multisecretary benchmark.

The observations \(Z_t\) can be simulated by independent internal randomization in the one-resource problem and convey no future-reward information.
Randomization cannot improve the optimal known-distribution value, so this value equals that of uniform multisecretary selection with capacity ratio \(1/2\).
The lower bound in \citet[Proposition~2]{Bray2025} is $\Omega(\log T)$ for that problem, which proves \eqref{eq:lower-bound}.
\end{proof}

\section{Parameter Dependence and Calibration}\label{app:constants}\label{sec:constants}
 This appendix records explicit bounds for the prescribed calibration and identifies the instance-dependent scales in the regret constant.

\subsection{Curvature and the early learning constant}
The constants \(D,K_0,K_T,A_0\) are defined in \eqref{eq:named-early}.
For \(T>4H\), Appendix~\ref{app:forward} proves the explicit bounds
\begin{equation}
 \sum_{t=1}^{M-1}E_t\le s_0D^2+K_T,
 \label{eq:explicit-prefix}
\end{equation}
\begin{equation}
 U_t\le\frac{6K_T}{\kappa^2T},\qquad
 E_t\le\frac{A_0}{t+s_0}+\frac{3K_T}{T},
 \qquad 1\le t\le M.
 \label{eq:explicit-pointwise}
\end{equation}

The bounds follow by telescoping the joint energy and depend polynomially on \(D,G,\kappa,\kappa^{-1}\), with \(\kappa=2/\mulo\).

\subsection{Complementary margins and the terminal charge}
Let $p_*:=\min_{i\in\B}y_i^*$ and $s_*:=\min_{i\in\N}s_i$, omitting a term when the corresponding set is empty. One explicit choice of the proof radii is
\begin{align}
 \rho&=\tfrac1{16}\min\{1,\sigma_0,p_*,s_* /(1+L)\},\notag\\
 \delta_B&=\tfrac18\min\{1,\dmin,\rho/\kappa\},\qquad
 \delta_i=\tfrac18\min\{1,\dmin,s_i\}.
 \label{eq:explicit-radii}
\end{align}
These values meet \eqref{eq:deltas}--\eqref{eq:rho} and
\(\delta_B<\rho/(4\kappa)\), as needed for the early weighted projection.
The analytical cutoff may be taken as
\begin{equation}
 H=1+\left\lceil\max\left\{s_0+2,\frac{2\kappa G}{\rho},
                     \frac{4\bar a}{3\dmin},8\right\}\right\rceil.
 \label{eq:explicit-H}
\end{equation}
 This choice satisfies \eqref{eq:H} through the exit-producing update. The cutoff and complementary margins are not algorithm inputs.

The Markov and maximal-inequality steps introduce inverse squares of $\rho$, $\delta_B$, $\delta_i$, and $\delta_i+s_i/4$. The restoring coefficients depend on $L,\mulo,\kappa$ as displayed in Lemma~\ref{lem:restoring}. The terminal block contributes at most a constant times $H$, and the bound for $T\le4H$ follows from $2\bar rT\le8\bar rH$.

The remaining estimates use sums, products, and inverse powers of these positive scales.
Hence the regret constant has a polynomial upper bound in the data scales (including dimension), \(L,\mu_0^{-1},\sigma_0^{-1}\), and the inverse positive price and slack margins at the reference solution.
Since \(\mu_0\le L\), \(\kappa^{-1}\) adds no independent parameter. The exponents and coefficients are not asserted to be sharp.

The bound is not uniform as \(p_*\) or \(s_*\) tends to zero and does not imply finite-horizon dominance of the conservative gain. Improving these constants or removing calibration requires further analysis.

\clearpage
\end{APPENDICES}
\end{document}